\documentclass{article}
\usepackage[utf8]{inputenc}
\usepackage{amsmath, amsfonts,amsthm}
\usepackage[left=1in,top=1in,right=1in,bottom=1in,letterpaper]{geometry}
\usepackage{multirow,booktabs,makecell}
\usepackage{graphicx,epstopdf,subcaption}
\usepackage{tikz}
\usetikzlibrary{positioning}
\usepackage{color,hyperref}
\usepackage[numbers]{natbib}
\usepackage{comment}
\usepackage{url}
\usepackage{enumitem}

\usepackage{authblk}

\usepackage[ruled,linesnumbered]{algorithm2e}

\newcommand{\bbR}{\mathbb{R}}

\newcommand{\calP}{\mathcal{P}}

\newcommand{\calF}{\mathcal{F}}
\newcommand{\calG}{\mathcal{G}}
\newcommand{\calJ}{\mathcal{J}}

\DeclareMathOperator*{\argmin}{argmin}

\DeclareMathOperator*{\Argmin}{Argmin}

\newcommand{\dd}{\mathrm{d}}

\usepackage{wrapfig}
\usepackage{longtable}
\usepackage{booktabs}
\usepackage{array}
\newtheorem{theorem}{Theorem}[section]
\newtheorem{lemma}[theorem]{Lemma}

\newtheorem{proposition}[theorem]{Proposition}

\theoremstyle{definition}
\newtheorem{definition}[theorem]{Definition}
\newtheorem{example}[theorem]{Example}

\theoremstyle{remark}

\title{High-dimensional Mean-Field Games \\by Particle-based Flow Matching}

\author[1]{Jiajia Yu}
\author[2]{Junghwan Lee}
\author[2]{Yao Xie}
\author[1]{Xiuyuan Cheng}

\affil[1]{{\small Department of Mathematics, Duke University}}
\affil[2]{{\small H. Milton Stewart School of Industrial and Systems Engineering, Georgia Institute of Technology.}}

\date{}

\begin{document}

\maketitle

\begin{abstract}
Mean-field games (MFGs) study the Nash equilibrium of systems with a continuum of interacting agents, which can be formulated as the fixed-point of optimal control problems. They provide a unified framework for a variety of applications, including optimal transport (OT) and generative models. Despite their broad applicability, solving high-dimensional MFGs remains a significant challenge due to fundamental computational and analytical obstacles. 
In this work, we propose a particle-based deep Flow Matching (FM) method to tackle high-dimensional MFG computation.
In each iteration of our proximal fixed-point scheme, particles are updated using first-order information, and a flow neural network is trained to match the velocity of the sample trajectories in a simulation-free manner. 
Theoretically, in the optimal control setting, we prove that our scheme converges to a stationary point sublinearly, and upgrade to linear (exponential) convergence under additional convexity assumptions.
Our proof uses FM to induce an Eulerian coordinate (density-based) from a Lagrangian one (particle-based), and this also leads to certain equivalence results between the two formulations for MFGs when the Eulerian solution is sufficiently regular.
Our method demonstrates promising performance on non-potential MFGs and high-dimensional OT problems cast as MFGs through a relaxed terminal-cost formulation.
\end{abstract}

\section{Introduction}

Mean-field games (MFGs)~\citep{lasry2007mfg,huang2006mfg} study the Nash Equilibria in games involving a continuum of indistinguishable, non-cooperative players. In an MFG, individual cost is affected by the aggregate behavior of the population $\rho$ due to interaction effects. 
Given the collective behavior of the population, every player seeks an optimal strategy $\hat{v}$ that minimizes their individual cost. However, as players update their strategies to $\hat{v}$, the overall population distribution also evolves to $\hat{\rho}$ induced by $\hat{v}$.
A mean-field Nash equilibrium (MFNE) in this context is a state where the strategy chosen by each player is optimal with respect to the population, and the population itself is consistent with these strategies. 

Mathematically, such an equilibrium can be formulated as the fixed point of an optimal control problem
$\Argmin_{(\tilde{\rho},\tilde{v})\in C_{(\rho,v)}} \calJ(\tilde{\rho},\tilde{v};\rho)$, 
where the cost  $\calJ$ as a functional of the players' strategy $\tilde v$ and individual distribution $\tilde \rho$ (induced by $\tilde v$) 
also involves an external population distribution $\rho$.
The equilibrium is then defined as 
\begin{equation}
\begin{aligned}
    (\hat{\rho},\hat{v}) \in & \Argmin_{(\tilde{\rho},\tilde{v})\in C_{(\rho,v)}} \calJ(\tilde{\rho},\tilde{v};\rho),\quad \rho=\hat{\rho},
\end{aligned}
\label{eq: mfg rho v}
\end{equation} 
that is, when the optimal control solution $\hat \rho$ and the population $\rho$ coincide. We refer to $(\hat{\rho},\hat{v})$ as the best response to $\rho$.
The constraint set $C_{(\rho,v)}$ essentially poses a continuity equation (CE), that is, $\tilde \rho$ satisfies the CE associated with $\tilde v$. The specifics of $C_{(\rho,v)}$ and $\calJ$ will be given in the subsequent sections. 
In this paper, we focus on MFGs with deterministic dynamics, i.e., the agents’ evolution is determined solely by their initial positions and controls, without any stochastic noise. These are also referred to as first-order MFGs, as the PDE defining the constraint set  $C_{(\rho, v)}$  is of first order.

Thanks to the mean-field approximation, MFGs provide a more tractable framework for analyzing systems of a large number of interacting agents. 
Since their introduction, MFGs have found applications across various domains, including economics~\citep{RL_MFG_eco_bookchapter,carmona2020MFG_fin_eco}, social sciences~\citep{MFC_epidemics}, and engineering~\citep{MFG_engineer_survey}.
More recently, MFGs have attracted growing attention due to their emerging connections with topics in machine learning, such as optimal transport~\citep{Benamou2000BBformulation}, normalizing flows~\citep{huang2023bridgingMFGandNF,MFG_generativemodels}, deep neural network training~\citep{ML_dynsystem}, and reinforcement learning~\citep{MFC_MARL,RL_MFG_eco_bookchapter}.
Motivated by applications, there is a growing interest in tackling the computation of MFGs, especially in high-dimensional space.

Solving MFGs encompasses and generalizes several key problems.
In the cost $\calJ(\tilde{\rho},\tilde{v};\rho)$, when the terms involving $\rho$ are the first variation of some potential functionals, the MFG problem~\eqref{eq: mfg rho v} admits a variational formulation (Proposition \ref{prop:MFC--MFG}). This class of problems is commonly referred to as potential MFGs or mean-field control (MFC) in the literature. On the other hand, when the cost $\calJ(\tilde{\rho},\tilde{v};\rho)$ is independent of $\rho$, the individual has no interactions with the population and the MFG problem~\eqref{eq: mfg rho v} reduces to an individual optimal control (OC) problem.
Notably, the dynamic optimal transport problem~\citep{Benamou2000BBformulation}, with a terminal constraint relaxed to Kullback–Leibler (KL) divergence, can be formulated as an MFG. This formulation in high dimensions is an important model for normalizing flows in generative modeling~\citep{huang2023bridgingMFGandNF,xu2025qflow}.

Despite substantial recent efforts, solving high-dimensional MFGs remains challenging.
Among others, one difficulty is due to their inherent fixed-point structure. 
In game theory, it is known that a na\"{i}ve fixed-point iteration $\rho^{(\ell+1)}=\hat{\rho}^{(\ell)}$ may not converge. A widely used alternative is {\it fictitious play}~\citep{cardaliaguet2017ficplay} $\rho^{(\ell+1)}:=(1-\alpha_{\ell})\rho^{(\ell)}+\alpha_{\ell}\hat{\rho}^{(\ell)}$, which converges for suitable $\alpha_{\ell}\in(0,1]$ under certain assumptions~\citep{yu2024ficplay} and has close connections to the generalized Frank-Wolfe algorithm~\citep{lavigne2023ficplaycondgrad}.
Despite its theoretical appeal, directly implementing fictitious play in high-dimensional settings becomes impractical because computing the exact best response is expensive.
Nevertheless, it suggests that even when the objective changes at each step, moving toward the best response of the current step eventually leads to an MFNE.

In this work, we propose an iterative neural method for high-dimensional MFGs that, at each step, jointly updates particle trajectories in Lagrangian coordinates and trains a Flow-Matching (FM) network to parameterize the velocity field. 
Specifically, we propose a proximal best response scheme motivated by fictitious play. Our approach limits the search to a local neighborhood around the current state instead of seeking a global best response and averaging at each iteration as done in classical fictitious play.
To implement this scheme, 
we introduce a concurrent particle and FM update. 
In each iteration, we first update a set of particles directly using the gradient-based information. 
Next, we train a neural flow model $v_{\theta}$ to match the velocity of the updated particle trajectories via a mean-squared loss as in FM ~\citep{albergo2023stointerp,lipman2023flowmatching,liu2022rectifiedflow}. 
While our setting differs from the original setup of FM,
where the endpoint distributions are given (accessed via finite samples) and the goal is to interpolate to the target distribution from a source distribution, the FM step in our method has the effect of disentangling the sampled particle trajectories and ensuring that the marginal sample distributions are preserved (see Figure \ref{fig:particle_ode_trajectories}). 
Theoretically, we prove the convergence rate of the proposed algorithm, and the property of FM to preserve marginal distribution is used in our analysis. In practice, the FM step allows $v_\theta$ to be trained over batches of sampled trajectories and accelerates convergence. 

We summarize our contribution as follows:
\begin{enumerate}
    \item We propose a proximal fixed-point scheme 
    that consists of 
    a particle optimization step
    and a Flow Matching step concurrently in each iteration
    to solve high-dimensional MFGs. 
    In particular, our scheme can solve potential MFG (including the optimal control setting) as well as general MFG by leveraging the fictitious play approach.    
    \item Theoretically, we prove that the particle-based proximal fixed-point scheme converges sublinearly and linearly with additional convexity assumptions.
    Our convergence rates hold in the optimal control setting,
    while some intermediate results apply to general MFGs. 
    In particular, our proof uses a ``trajectory disentangling'' property of the flow matching, which also allows us to obtain 
    certain equivalence between the Lagrangian (particle-based) and Eulerian (density-based) formulations, assuming that the Eulerian solution is sufficiently regular. 
    
    \item We apply the proposed algorithm to simulated and image datasets,
    including a non-potential MFG
    and the image-to-image translation task formulated as a relaxed OT problem as an MFG. 
\end{enumerate}

\begin{figure}[t]
    \centering
    \begin{subfigure}[t]{0.485\textwidth}
        \centering
        \includegraphics[height=4.5cm]{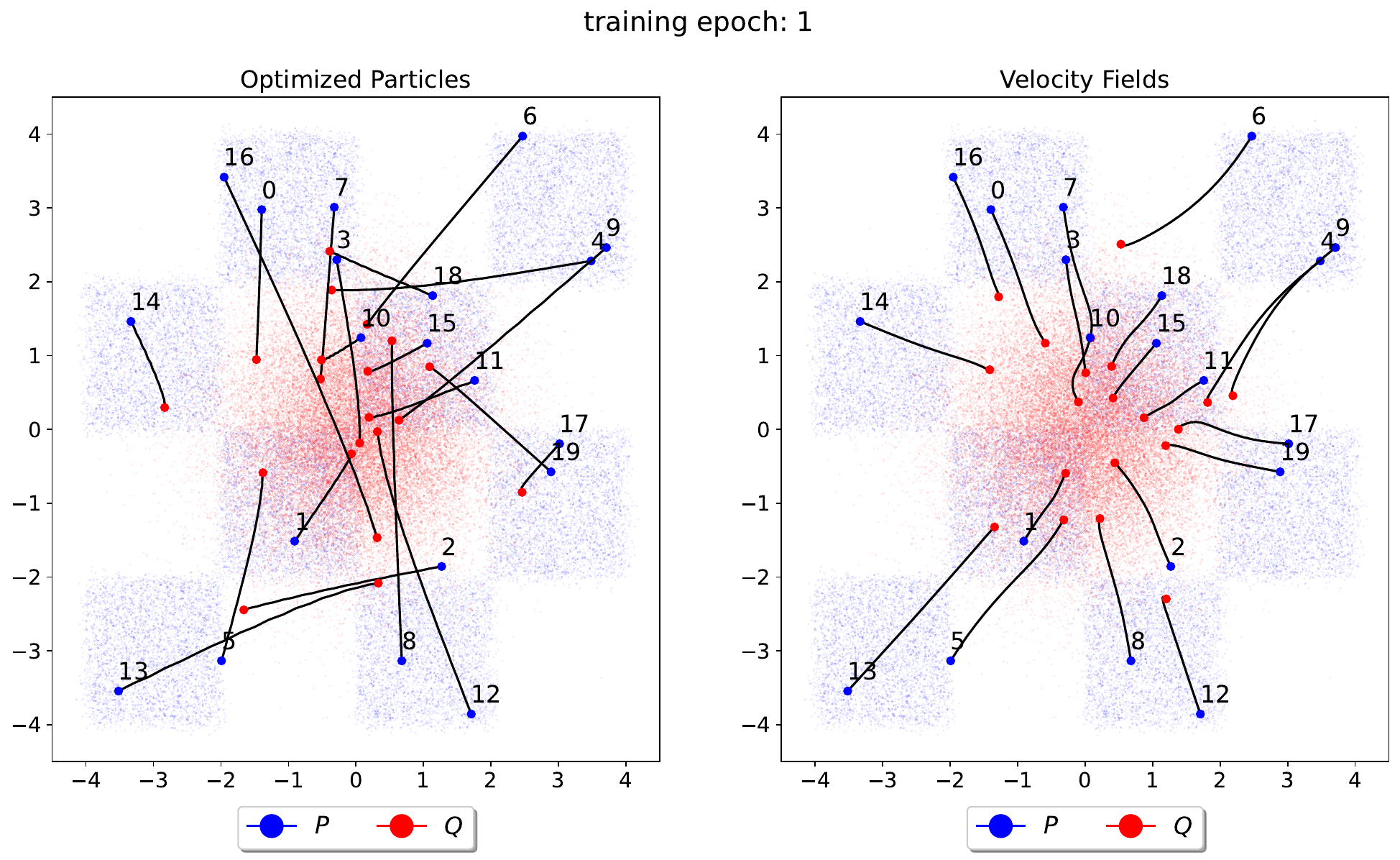}
        \caption{} 
        \label{fig:ode_traj_epoch1}
    \end{subfigure}
    \begin{subfigure}[t]{0.485\textwidth}
        \centering
        \includegraphics[height=4.5cm]{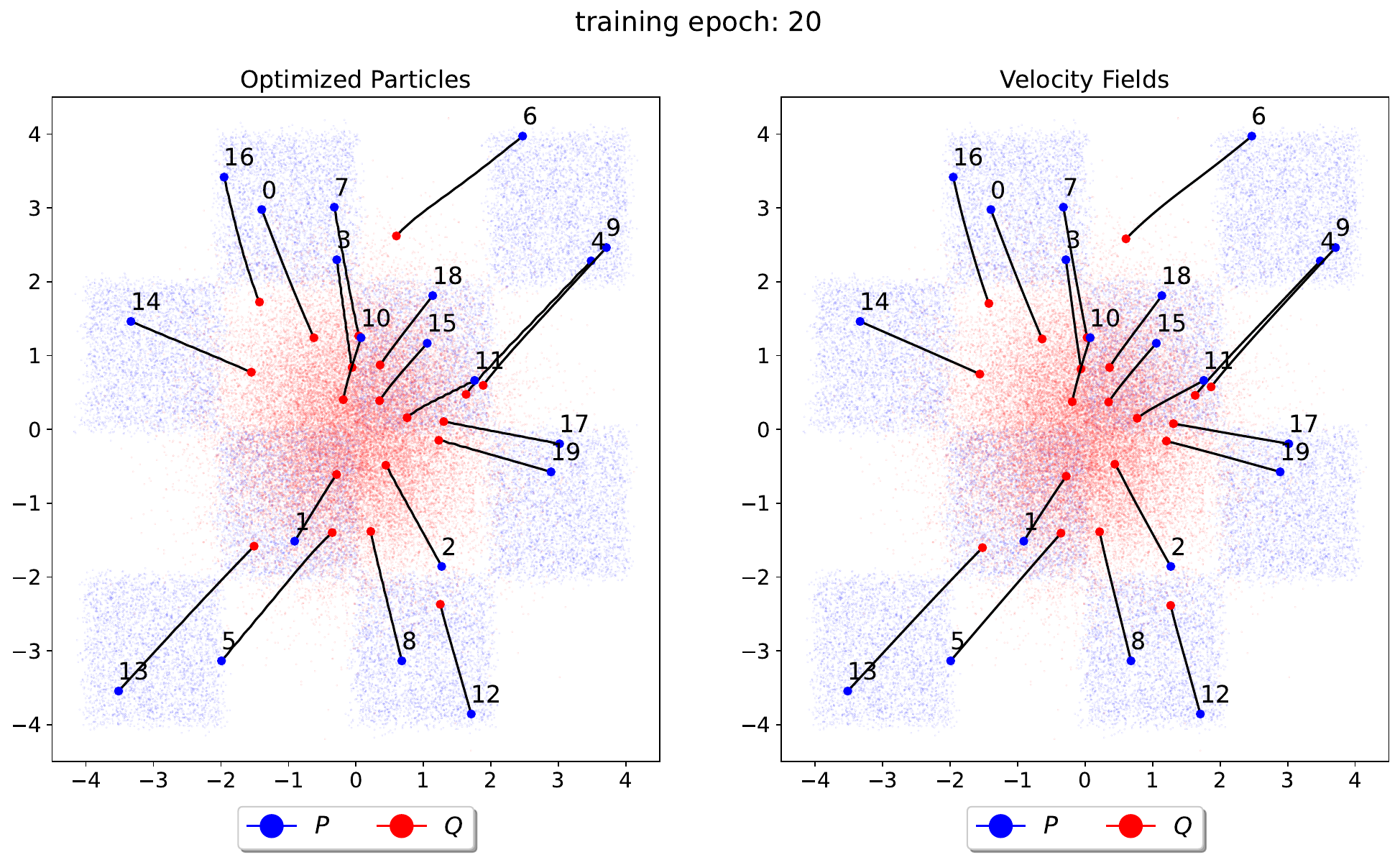}
        \caption{} 
        \label{fig:ode_traj_epoch20}
    \end{subfigure}
    \vspace{-5pt}
    \caption{
    Illustration of the trajectory disentangle effect by Flow Matching (FM). 
    (a) and (b) show the sample trajectories at the beginning of training and after 20 epochs of training, respectively,
    where ``optimized particles'' stands for trajectories after a particle update,
    and ``velocity field'' shows trajectories resampled from a learned velocity field by FM, see Algorithm \ref{alg: gradient descent}.
    Theoretically, we prove that FM disentangles the trajectories and reduces the dynamic cost, while leaving the interaction and terminal costs unchanged.}
    \label{fig:particle_ode_trajectories}
\end{figure}

\subsection{Related works}

\paragraph{Fictitious play in solving MFGs}

Fictitious play, first introduced in~\citep{brown1949ficplay,brown1951ficplay}, is a classical algorithm in game theory. 
Leveraging PDE analysis, ~\citep{cardaliaguet2017ficplay} applied the strategy to MFG 
and proved convergence for MFC, namely the class of MFG with variational formulation.
The fictitious play is related to the generalized Frank–Wolfe algorithm, 
as pointed out in \citep{lavigne2023ficplaycondgrad},
and the work furtherly used this connection to accelerate the algorithm and also proved a faster convergence rate in the MFC setting. 
These results were extended to general MFGs in the recent work of~\citep{yu2024ficplay}, which obtained similar acceleration without relying on a variational structure.
Computation-wise, the implementations in~\citep{lavigne2023ficplaycondgrad,yu2024ficplay} are based on the PDE formulation, and specifically, via the optimality conditions of the optimization problem in~\eqref{eq: mfg rho v}, which results the coupled system of Hamilton–Jacobi–Bellman and Fokker–Planck equations. The algorithms therein utilize mesh-based solvers, which achieve numerical accuracy in low-dimensional settings but encounter challenges in scaling to high dimensions.

Besides PDE-based implementations, other works exploited the connection between MFGs and reinforcement learning (RL)~\citep{lauriere2020ficplay_finitestate,guo2023general},
and used RL techniques to implement fictitious play. 
These RL methods typically assume a finite or discrete state space, such as a graph, and may not be efficiently applied to continuous spaces like $\mathbb{R}^d$.
For example, if one adopts a mesh-based discretization, then the algorithm would still face the curse of dimensionality when $d $ is large. 
We note that most of these works consider an infinite time horizon, in which case the problem can be reduced to finding a time-invariant equilibrium. 
In contrast, the finite-horizon case requires solving a dynamic problem that evolves over time and thus is a more challenging problem. In this work, we focus on the finite time horizon rescaled to $[0,1]$. 

\paragraph{Flow-based methods and simulation-free methods}
To avoid discretizing the state space, recent works use neural networks to represent the trajectories of particles and define the objective directly in terms of network parameters. For example,~\citep{Ruthotto2020pnas,Zhou2025FBODEforMFC,assouli2025initialization} parameterize the value function with a neural network; the particle trajectories are then induced by the value function via ODEs. \citep{huang2023bridgingMFGandNF} uses discrete normalizing flows to model particle transitions between time steps. 
A common limitation of these methods is that they assume a variational formulation of MFGs, which only exists in special cases (e.g., potential MFGs or MFC). Even in those cases, the objective depends nonlinearly on the population distribution $\rho$, which is coupled with the velocity field $v$ through a PDE constraint. As a result, training requires backpropagating through PDE or ODE solvers, which is computationally expensive.

Simulation-free methods avoid solving PDEs or ODEs during the forward process of a neural network, and therefore simplify backpropagation and lead to significantly lower computational cost.
These methods have recently been developed for stochastic optimal control (SOC), where the cost functions $F$ and $G$ are independent of the population distribution $\rho$, and the dynamics include stochastic noise.
\citep{hua2025ICLRSimFreeSOC} derive analytical gradient expressions that avoid simulation and do not decouple the distribution from the control. Other works~\citep{domingo-enrich2024SOCmatching,domingo-enrich2025adjointmatching} use a decoupling approach, training the flow through control matching or adjoint equations.
In comparison, our method addresses the first-order MFGs, which is a broad class of problems, and in particular, allows $F$, $G$ to depend on $\rho $. 
Conceptually, our approach leverages the Lagrangian coordinates formulation of MFGs and trains the flow velocity field by directly matching to sample velocities.

\paragraph{Deep network methods for MFGs}

Beyond the works mentioned above, several other deep learning approaches have been developed for solving MFGs. \citep{Lin2021alternating,Gomes2023dualascent} formulate the problem as a minimax optimization and alternate updates between two variables to find a saddle point. \citep{chen2023hybrid,zhou2025ACMFG,fouque2025ACMFG} adapt actor-critic methods from reinforcement learning to the MFG setting. \citep{chen2023PIMFG} uses the PDE formulation of MFGs and designs two neural network modules to force the HJB and FP together.
Our approach adopts flow matching techniques in modern deep generative models, and we demonstrate the efficiency of our method on high-dimensional image data.

\subsection{Notations} 
We summarize the notations in Table~\ref{tab:notation} (Appendix~\ref{apsec:notation}). 
A hat~$\hat{\cdot}$ denotes best response quantities. 
Let $\mathcal{P}_2(\mathbb{R}^d)$ be the space of Borel probability measures with finite second moments. We use both the $L^2$ and Wasserstein-2 ($W_2$) metrics on $\mathcal{P}_2(\mathbb{R}^d)$, specifying which is used when needed.
For $\mu \in \mathcal{P}_2(\mathbb{R}^d)$, we define $L^2(\mu)$ as the Hilbert space of vector fields $\xi: \mathbb{R}^d \to \mathbb{R}^d$ with $\int \|\xi(x)\|^2 \dd\mu(x) < \infty$, inner product $\langle \xi, \eta \rangle_{\mu}:=\int \xi(x)\cdot\eta(x)\dd\mu(x)$, and the induced norm $\|\cdot\|_{\mu}$. 
Similarly, $L^2(\mu \otimes [0,1])$ denotes time-dependent vector fields $\xi: \mathbb{R}^d \times [0,1] \to \mathbb{R}^d$ with $\int_0^1\|\xi_t\|_{\mu}^2\dd t <\infty$, inner product $\langle \xi, \eta \rangle_{\mu \otimes [0,1]}:=\int_0^1\langle\xi_t,\eta_t\rangle_{\mu}\dd t$ and induced norm $\|\cdot\|_{\mu \otimes [0,1]}$.

The symbol $\rho$ may denote a static distribution in $\mathcal{P}_2(\mathbb{R}^d)$ or a time-dependent mapping $[0,1] \to \mathcal{P}_2(\mathbb{R}^d)$, with $\rho_t$ denoting the distribution at time $t$. Similar conventions apply to $v$ and $X$, which may refer to functions $\mathbb{R}^d \to \mathbb{R}^d$ or time-dependent maps $\mathbb{R}^d \times [0,1] \to \mathbb{R}^d$.
Constraint sets are denoted by $C_{\cdot}$, where the subscript indicates the constrained variable. 
Objective functionals are denoted by $\mathcal{J}$. For instance, $\mathcal{J}(\tilde{\rho}, \tilde{v}; \rho)$ evaluates the cost of a distribution-velocity flow under a fixed reference flow $\rho$, and $\mathcal{J}(X; \rho)$ evaluates the cost of a characteristic map under a fixed reference flow $\rho$.

\section{Preliminaries}
\label{sec: pre}
In this section, we begin by reviewing the Eulerian and Lagrangian formulations of flow dynamics and their relations.
We then review the formulation of mean-field games (MFGs), along with two important special cases: potential MFGs or mean-field control problems, and optimal control problems. 
Finally, we summarize the fictitious play algorithm and its convergence.

\subsection{Eulerian and Lagrangian coordinates}

Let $[0,1]$ be the time interval and $\mu\in\calP_2^r(\bbR^d)$ be a fixed initial distribution that has density, where $\calP_2^r(\bbR^d) = \calP_2(\bbR^d) \cap AC(\bbR^d) $.
In MFGs, the distribution evolution $\tilde{\rho}$ is induced by a control flow $\tilde{v}$ via the continuity equation. 
We restrict our focus to distribution-control pairs $(\tilde{\rho},\tilde{v})$ in the constraint set 
\begin{equation}
    C_{(\rho,v)}:=\left\{(\tilde{\rho},\tilde{v}):  \partial_t\tilde{\rho}+\nabla\cdot(\tilde{\rho} \tilde{v})=0,\, \tilde{\rho}(\cdot,0)=\mu,\, \tilde{\rho}\in C_{\rho},\, \tilde{v}_t\in L^2(\tilde{\rho}_t)
    \right\}.
\label{eq: cstr set of rho,v}
\end{equation}
Here $
C_{\rho}=AC^2(0,1;(\calP_2(\bbR^d),W_2))$ is the set of absolutely continuous curves in Wasserstein space (Definition~\ref{def:ac curve} in the appendix).
According to~\citep[Thm.~8.3.1]{AGS_book_gradientflows}, for any $\rho \in C_\rho$, there exists a Borel vector field $v$ such that $v_t \in L^2(\rho_t)$ and the pair $(\rho_t, v_t)$ satisfies the continuity equation in the sense of distributions.
From a fluid dynamics perspective, $v(x,t)$ describes the motion of fluid by observing the change at a fixed spatial location over time and is known as the {\it Eulerian} coordinate. It is analogous to monitoring the population density at a fixed state $x$ over time.
Due to its reliance on a fixed spatial grid or mesh, Eulerian simulation becomes impractical in high dimensions.

Alternatively, one may adopt a particle-based perspective by considering $X\in C_X:=H^1(0,1;L^2(\mu))$.
Here, $X(x,\cdot):[0,1]\to\mathbb{R}^d$ represents the trajectory of an individual agent starting from position $x$ with $X(x,0)=x$.
For any a.e. $t$, the weak derivative $\partial_t\tilde{X}(\cdot,t)$ exists and is in $L^2(\mu)$.
When all particles evolve under the same velocity field, the Eulerian and Lagrangian descriptions are linked by the ordinary differential equation:
\begin{equation}
    \partial_t{X}(x,t) = v(X(x,t),t).
\label{eq: Lagrangian ODE}
\end{equation}
and $X$ is known as the {\it Lagrangian} coordinates.

Let $C_{(X,v)}$ be the constraint set where the ODE is satisfied for $\mu$-a.e. $x$ and a.e. $t$:
\begin{equation}
\label{eq: C_X,v}
(\tilde{X},\tilde{v})\in C_{(X,v)}:= \left\{ (\tilde{X},\tilde{v}): \partial_t\tilde{X}(x,t) = \tilde{v}(\tilde{X}(x,t),t),\, \tilde{X}(x,0)=x,\, \tilde{X}\in C_X  \right\}.   
\end{equation}
Notice that for $X\in C_X$, the trajectories may intersect, and there may not exist $v$ such that $(X,v)\in C_{(X,v)}$. In Section~\ref{sec:theory}, we show that by flow matching, there exist $(\rho,v)\in C_{(\rho,v)}$ such that $v$ takes average on intersecting points and $\rho_t=(X_t)_\#\mu$.
On the other hand, given $(\rho,v)\in C_{(\rho,v)}$, the ODE solution $X$ to $v$ may not always exist and may not be unique if it exists.
Let $C_v$ be the set of functions $v:\bbR^d\times[0,1]\to\bbR^d$ that are bounded and Lipschitz continuous in $x$ on every compact set $B\subset\bbR^d$. Then, by a standard existence and uniqueness theorem for ODEs (see, for example, \citep[Lemma 8.1.4]{AGS_book_gradientflows}), there exists a unique solution $X$ such that $(X,v)\in C_{(X,v)}$ and $(X_t)_\#\mu=\rho_t$.

In the studies of dynamic OT and application to optimal control, a useful extension of the classical ODE flow is via measures on the path space (Superposition Principle \citep{Ambrosio2008superposition, AGS_book_gradientflows}), which, e.g., allows branching of trajectories that follow the velocity field $v$ in a weaker sense \cite{Lisini2006ACcurveinWspace,cavagnari2022LEK}.
In this paper, we will parametrize the velocity field $v$ by a neural network, which in practice usually gives a regular $v$. 
Thus, we focus on when $v$ is regular, and specifically when the solution to MFG \eqref{eq: mfg rho v} has $v \in C_v$.

\subsection{Mean-field games}

Given a population distribution evolution $\rho \in C_{\rho}$, the individual cost of adopting strategy $(\tilde{\rho},\tilde{v})$ depends on $\rho$. 
Precisely, let $F, G : \calP_2(\bbR^d) \times \bbR^d \to \bbR$ be operators acting on the population measure. For $\rho_t \in \calP_2(\bbR^d)$, $F[\rho_t] : \bbR^d \to \bbR$ is a function on $\bbR^d$. It assigns an interaction cost $F[\rho_t](x)$ to the individual at position $x$ at time $t$ given the population distribution $\rho_t$.
The total individual cost in response to $\rho$ is defined by
\begin{equation}
    \calJ(\tilde{\rho},\tilde{v};\rho):= \int_0^1\int  \left(\frac{1}{2}\|\tilde{v}(x,t)\|^2 +  F[\rho_t](x) \right) \dd \tilde{\rho}_t(x)\dd t + \int G[\rho_1](x) \dd \tilde{\rho}_1(x).
\label{eq: energh J(rho,v;rho)}
\end{equation}
A pair $(\hat{\rho}, \hat{v})$ that minimizes $\calJ(\tilde{\rho}, \tilde{v}; \rho)$ over $(\tilde{\rho}, \tilde{v})\in C_{(\rho, v)}$ is called the best response to $\rho$.
The objective in an MFG is to solve~\eqref{eq: mfg rho v} to find an MFNE, which is a triple $(\hat{\rho},\hat{v};\rho)$ such that the individual best response is consistent with the population, i.e., $\hat{\rho}=\rho$.
In the objective~\eqref{eq: energh J(rho,v;rho)}, we refer to $F$ as the interaction coupling and $G$ as the terminal coupling, where $F[\rho_t]$ and $G[\rho_1]$ represent the interaction cost and terminal cost, respectively. The couplings $F,G$ are typically either local, of the form $F[\rho](x)=f(p(x))$ with $f:\bbR_{\geq 0}\to\bbR$ and $p$ being the density of $\rho$, or nonlocal of form $F[\rho](x)=(K_x*\rho)(x)$ with $K_x$ being the kernel at $x$.

When the couplings $F,G$ are $L^2$ first variations of functionals $\calF,\calG$ on $\calP_2(\bbR^d)$, \eqref{eq: energh J(rho,v;rho)} takes the form as 
\begin{equation}
    \calJ(\tilde{\rho},\tilde{v};\rho)
    := \int_0^1\int 
        \left( \frac{1}{2}\|\tilde v(x,t)\|^2 
        + D\calF[\rho_t](x)\right)
        \dd  \tilde \rho_t(x)\dd t
        + \int  D\calG[\rho_1](x) \dd \tilde{\rho}_1(x),
\label{eq: energh J(rho,v;rho)-potential}
\end{equation}
and the formal definition of the derivatives $D\calF$, $D\calG$ can be found in the appendix (Definition~\ref{def:derivative}).
In this case, the MFG problem~\eqref{eq: mfg rho v} is associated to a variational form \citep{lasry2007mfg,Achdou2021MFGbook}. 
We summarize the result in the following proposition, and include a proof in Appendix~\ref{apsec: proof MFC--MFG} for completeness. 
\begin{proposition}
\label{prop:MFC--MFG}
Let $\calF:\calP_2(\bbR^d)\to\bbR\cup\{+\infty\}$ and $\calG:\calP_2(\bbR^d)\to\bbR^d\cup\{+\infty\}$ be proper and consider the following optimization problem:
\begin{equation}
\begin{aligned}
    \inf_{(\rho,v)\in C_{(\rho,v)}} &\quad \calJ(\rho,v):= \int_0^1\int \frac{1}{2}\|v(x,t)\|^2\dd \rho_t(x)\dd t + \int_0^1 \calF(\rho_t)\dd t + \calG(\rho_1),\\
\end{aligned}
\label{eq: mfc rho v}
\end{equation}
If $\calF,\calG$ have first variations $F,G:\calP_2(\bbR^d)\times\bbR^d\to\bbR\cup\{+\infty\}$ under $L^2$ metric, 
and $(\hat{\rho},\hat{v})\in C_{(\rho,v)}$ is the minimizer to~\eqref{eq: mfc rho v}, 
then $(\hat{\rho},\hat{v})$ solves the fixed-point problem~\eqref{eq: mfg rho v}.
If, in addition, $\calF,\calG$ are convex in $\rho$ under $L^2$ metric, then $(\hat{\rho},\hat{v})$ is the minimizer to~\eqref{eq: mfc rho v} if and only if $(\hat{\rho},\hat{v})$ solves the fixed-point problem~\eqref{eq: mfg rho v}.
\end{proposition}

The key for the proposition to hold is that, up to the change of variables $(\rho, m) = (\rho, \rho v)$, the continuity constraint becomes linear and the dynamic cost becomes convex.  
With this change of variable, it becomes straightforward that the fixed-point formulation~\eqref{eq: mfg rho v} is equivalent to the first-order optimality condition of the variational problem~\eqref{eq: mfc rho v}.

The variational problem~\eqref{eq: mfc rho v} is referred to as a potential MFG or mean-field control (MFC) problem.
$\calF$ and $\calG$ are referred to as interaction (potential) cost and terminal (potential) cost.
The variational structure facilitates the usage of many optimization algorithms and connects the MFG~\eqref{eq: mfg rho v} to several important optimization problems.

For example, when the interaction cost is removed and the terminal cost is a Kullback–Leibler (KL) divergence, the MFC \eqref{eq: mfc rho v} reduces to a normalizing flow with optimal transport regularization, which is widely used in generative modeling~\citep{Onken2021OTFlow,huang2023bridgingMFGandNF}.
Furthermore, since the KL divergence is convex, solving the MFG \eqref{eq: mfg rho v} is sufficient to solve the regularized normalizing flow \eqref{eq: mfc rho v}.
In the case of zero interaction cost and general terminal costs $\calG$, the MFC \eqref{eq: mfc rho v} recovers the Jordan–Kinderlehrer–Otto (JKO) scheme~\citep{jordan1998JKO} for Wasserstein gradient flows, which has applications in both generative modeling~\citep{xu2023JKO-iFlow} and solving high-dimensional kinetic equations~\citep{huang2024JKOforLandau}.

When $F$ and $G$ are functions on $\bbR^d$ and independent of $\rho$, the game is potential with $\calF(\rho)=\int F(x)\dd\rho(x)$ and $\calG(\rho)=\int G(x)\dd\rho(x)$ linear in $\rho$, and both \eqref{eq: mfg rho v} and \eqref{eq: mfc rho v} reduce to a single player optimal control (OC) problem. 
The costs $\calF$ and $\calG$ being linear in $\rho$ makes it convenient to approximate them with expectations.

\subsection{Fictitious play}

Fictitious play is a classical algorithm in game theory, first introduced in~\citep{brown1951ficplay} and later adapted to MFGs in~\citep{cardaliaguet2017ficplay}. 
The update rule is given by:
\begin{equation}
\left\{\begin{aligned}
    &(\hat{\rho}^{(\ell)},\hat{v}^{(\ell)}) \in \Argmin_{(\tilde{\rho},\tilde{v})\in C_{(\rho,v)}}\calJ(\tilde{\rho},\tilde{v};\rho^{(\ell)}),\\
    &\rho^{(\ell+1)} = (1-\alpha_{\ell})\rho^{(\ell)} + \alpha_{\ell} \hat{\rho}^{(\ell)}.
\end{aligned}\right.
\label{eq: ficplay}
\end{equation}
When $\alpha_{\ell}=1$, the update becomes a fixed-point iteration, which is shown to potentially diverge~\citep[Sec. 5.1]{yu2024ficplay}.
For potential MFGs,~\citep{cardaliaguet2017ficplay} proved the convergence using a diminishing weight $\alpha_{\ell}=\frac{1}{\ell}$. Further,~\citep{lavigne2023ficplaycondgrad} showed that fictitious play is equivalent to the generalized Frank-Wolfe algorithm for potential MFGs, and proved a convergence rate of $\mathcal{O}(\ell^{-p})$ for $\alpha_{\ell}=\frac{p}{\ell+p}$ (with $p>0$). More recently, \citep{yu2024ficplay} extended the same convergence rate to general (non-potential) MFGs and highlighted that the best choice of the weight $\alpha_{\ell}$ depends on the local convexity of the dynamic cost.

\section{Method}
\label{sec:method}

We first reformulate the MFG problem \eqref{eq: mfg rho v} in Lagrangian coordinates and express the objective as an expectation.
We then propose a proximal fixed-point algorithm that alternates between updating the particles $X$ and the velocity field $v$.
Our method utilizes optimization over sample particle trajectories, namely ``particle-based'', which is more scalable to high-dimensional spaces.

\subsection{Reformulation in Lagrangian coordinates}

Consider characteristic maps $X:\bbR^d\times[0,1]\to\bbR^d$ in the set $C_X$.
For $\tilde{X}\in C_X$, $\tilde{X}(x,\cdot)$ denotes the trajectory of a sample particle starting from $x$.
For a given $\rho\in C_{\rho}$, define the individual control objective in MFG as
\begin{equation}
\label{eq: obj mfg X}
\calJ(\tilde{X};\rho):= \mathbb{E}_{x\sim\mu}\left[\int_0^1  \left(\frac{1}{2}\|\partial_t\tilde{X}(x,t)\|^2 
    + F[\rho_t](\tilde{X}(x,t)) \right) \dd t 
    + G[\rho_1](\tilde{X}(x,1))\right],\\   
\end{equation}
and population objective in potential MFG as
\begin{equation}
\label{eq: obj mfc X}
\calJ(\tilde{X})
    :=\mathbb{E}_{x\sim\mu}\left[\int_0^1\frac{1}{2}\|\partial_t\tilde{X}(x,t)\|^2\dd t \right]
    + \int_0^1\calF((\tilde{X}_t)_\#\mu)\dd t 
    + \calG((\tilde{X}_1)_\#\mu).
\end{equation}
Using the Lagrangian coordinates, the MFG \eqref{eq: mfg rho v} can be formulated as
\begin{equation}
\label{eq: mfg X}
    \hat{X} \in \Argmin_{\tilde{X}\in C_X} 
    \calJ(\tilde{X};\rho),\quad \rho_t=\hat{\rho}_t:=(\hat{X}_t)_{\#}\mu.
\end{equation}
Similarly, the potential MFG \eqref{eq: mfc rho v} becomes
\begin{equation}
\label{eq: mfc X}
\min_{\tilde{X}\in C_X} \calJ(\tilde{X}).
\end{equation}
Notice that in the reformulation, we only require $\tilde{X} \in C_X$, that is, each trajectory itself is a differentiable curve, but among each other the trajectories may intersect. In other words, for general $\tilde{X} \in C_X$, there may not exist a flow $\tilde{v}$ such that $(\tilde{X},\tilde{v})\in C_{(X,v)}$. 
In section \ref{sec:theory}, we will resolve this issue using Flow Matching (FM): FM will provide a velocity field $\tilde v$ such that the marginal density $\tilde \rho$ of $\tilde X$ satisfies the continuity equation associated with $\tilde v$. In addition, this pair $(\tilde \rho, \tilde v)$ will preserve or lower the objective $\calJ$ compared to the objective on $\tilde X$ (Lemma \ref{lem: X to rho v}).
This result will further allow us to establish certain equivalence between the Eulerian and Lagrangian problems under regularity conditions.

\subsection{Particle optimization and neural flow matching}

A direct adaptation of fictitious play to the Lagrangian formulation~\eqref{eq: mfg X} suggests computing the best response $\hat{X}^{(\ell)}$ to $\rho^{(\ell)}$ and then update by $X^{(\ell+1)}=(1-\alpha_{\ell}) X^{(\ell)}+\alpha_{\ell}\hat{X}^{(\ell)}$ and $\rho^{(\ell+1)}_t=(X^{(\ell+1)}_t)_{\#}\mu$. 
However, solving for the best response can be computationally expensive. Moreover, when the step size $\alpha_{\ell}$ is small, an accurate best response is often unnecessary, as the effect of any approximation error is scaled down by $\alpha_{\ell}$.

By definition, the best response provides a descent direction for the current objective $\calJ(\cdot;\rho^{(\ell)})$.
A key observation from fictitious play is that convergence to a fixed point is possible as long as each update yields a small improvement relative to the current objective, even though the objective itself changes at each iteration.
Motivated by this, we search for the proximal best response at each iteration and propose a proximal fixed-point scheme.
Specifically, given $v^{(\ell)}$, we pick a step size $\alpha_{\ell}$ and do the following three steps:
\begin{enumerate}
    \item Resample and cost computation: solve for $X^{(\ell)}$ and $\rho^{(\ell)}$ and compute $F[\rho^{(\ell)}_t],G[\rho^{(\ell)}_1]$, where  
    \begin{equation}
        (X^{(\ell)},v^{(\ell)})\in C_{(X,v)},\quad
        (\rho^{(\ell)},v^{(\ell)})\in C_{(\rho,v)}.
    \label{eq: update X rho}
    \end{equation}
    
    \item Particle update: update $X$ by proximal descent:
    \begin{equation}
    \begin{aligned}
        X^{(\ell+1/2)} = \argmin_{X\in C_X} 
        ~\calJ(X;\rho^{(\ell)}) + \frac{1}{2\alpha_{\ell}}\left(\|X-X^{(\ell)}\|_{\mu\otimes [0,1]}^2 
        + \|X_1-X_1^{(\ell)}\|_{\mu}^2\right) . \\
    \end{aligned}
    \label{eq: update X}
    \end{equation}

    \item Flow matching update: update $v$ by regression
    \begin{equation}
        v^{(\ell+1)}:=\argmin_{v} \int_0^1\int \left\|v(X^{(\ell+1/2)}(x,t),t) - \partial_tX^{(\ell+1/2)}(x,t)\right\|^2\dd\mu(x)\dd t
    \label{eq: update v}
    \end{equation}

\end{enumerate}

In Lemmas~\ref{lem: descent by X update} and~\ref{lem: descent by v update}, we prove that both the particle update~\ref{eq: update X} and the flow matching update~\ref{eq: update v} are well-defined.
When $X^{(\ell)}$ is not a stationary point, the particle update yields a decrease in the objective function.
More importantly, if $X^{(\ell+1/2)}$ is not induced by a velocity field $v$, then the subsequent flow matching and resampling steps also result in a decay of the objective.
Based on this observation, we perform the flow matching update less frequently in practice, as it contributes to improvement primarily when particle trajectories intersect.

\section{Theory}
\label{sec:theory}

We begin by noting that the original problem is posed in the Eulerian formulation~\eqref{eq: mfg rho v}, while our algorithm operates in the Lagrangian formulation~\eqref{eq: mfg X}. 
In this section, we first derive a certain equivalence between the two formulations in Theorem~\ref{thm: euler=lagrangian}. 
In Section \ref{subsec:cvg}, we analyze the convergence of the proposed scheme~\eqref{eq: update X rho}\eqref{eq: update X}\eqref{eq: update v}, proving its descent property under for general MFGs
and its convergence rates in the optimal control setting. 

\subsection{Equivalence between Eulerian and Lagrangian formulations}

To prove Theorem~\ref{thm: euler=lagrangian}, we first relate $\tilde{X}\in C_X$ to a corresponding pair $(\tilde{\rho},\tilde{v})\in C_{(\rho,v)}$ and compare their associated costs.
Lemma~\ref{lem: X to rho v} shows that for any $\tilde{X}\in C_X$, there is a unique pair $(\tilde{\rho},\tilde{v})\in C_{(\rho,v)}$ obtained by flow matching. In addition, the cost of $\tilde{X}$ is bounded below by the cost of $(\tilde{\rho},\tilde{v})$. 
In the other direction, Lemma~\ref{lem: rho v to X} establishes that the equality holds when $\tilde{X}$ is induced by a velocity field $\tilde{v}\in C_v$.
These two lemmas will also be used in the convergence analysis in Section \ref{subsec:cvg}, and specifically, to prove Lemma \ref{lem: descent by v update}.
The proofs of Lemmas~\ref{lem: X to rho v} and~\ref{lem: rho v to X} are given in Appendices~\ref{apsec: proof lem X to rho v} and~\ref{apsec: proof lem rho v to X}, respectively.

\begin{lemma}
\label{lem: X to rho v}
    Let $\tilde{X}\in C_X$ and set $\tilde{\rho}_t:=(\tilde{X}_t)_{\#}\mu$, $\tilde{v}=\argmin_{v} \int_0^1\int\|v(\tilde{X}(x,t),t)-\partial_t\tilde{X}(x,t)\|^2\dd\mu(x)\dd t$. 
    Then, $\tilde{\rho}\in C_{\rho}$, $(\tilde{\rho},\tilde{v})\in C_{(\rho,v)}$, and $\calJ(\tilde{\rho},\tilde{v};\rho)\leq\calJ(\tilde{X};\rho)$ for any $\rho\in C_{\rho}$.
\end{lemma}

\begin{lemma}
\label{lem: rho v to X}
    If $(\tilde{\rho},\tilde{v})\in C_{(\rho,v)}$ and $\tilde{v}\in C_v$, then there exists $\tilde{X}\in C_X$ satisfying $(\tilde{X},\tilde{v})\in C_{X,v}$ and for any $\rho\in C_{\rho}$,
    $\calJ(\tilde{X};\rho)=\calJ(\tilde{\rho},\tilde{v};\rho)$.
\end{lemma}

As a direct result of Lemmas~\ref{lem: X to rho v} and~\ref{lem: rho v to X}, the following theorem shows the relationship between Eulerian coordinates and the Lagrangian coordinates. 
Specifically, under certain regularity conditions,
 we show that for general MFGs, an Eulerian solution induces a Lagrangian solution;
For potential MFGs, the optimal value of the Eulerian and Lagrangian formulation coincides, and a Lagrangian solution also induces an Eulerian solution by flow matching.
The proof is in Appendix~\ref{apsec: proof thm euler=lagrange}.

\begin{theorem}
\label{thm: euler=lagrangian}

(i) Suppose $(\rho^*,v^*)$ is a solution to the MFG~\eqref{eq: mfg rho v} and $v^*\in C_v$, then
$v^*$ induces characteristics $X^*\in C_X$ that is a solution to~\eqref{eq: mfg X}.

(ii)
Suppose the potential MFG ~\eqref{eq: mfc rho v} admits a solution $(\rho^*,v^*)$ with $ v^*  \in C_v$. 
Then $v^*$ induces characteristics $X^*\in C_X$ that is a solution to~\eqref{eq: mfc X} and $\calJ(X^*)=\calJ(\rho^*,v^*)=: J^*$.
In addition, for any other $X\in C_X$ that is also a solution to~\eqref{eq: mfc X}, let $(\rho, v)$ be the flow-matched Eulerian representation of $X$ as in Lemma \ref{lem: X to rho v},
then $(\rho, v)$ is a solution to \eqref{eq: mfc rho v}, and in this case $J(X) = J(\rho, v) = J^*$.
\end{theorem}

For control problems, the equivalence between the Lagrangian and Eulerian formulations has been previously studied in~\citep{cavagnari2022LEK} by considering distributions on path space, and their Lagrangian formulation is in a weak sense. 
In contrast, our framework is motivated by neural network methodology, and we focus our theory on classical ODE flows where trajectories from $\mu$-a.e.\,$x$ are well-defined, which is possible when assuming the MFG solution has a regular $v^*$. 

\subsection{Convergence rate of the proximal fixed-point scheme}
\label{subsec:cvg}

In our scheme, the particle update~\eqref{eq: update X} reduces the total cost;
The flow-matching update~\eqref{eq: update v} and resampling~\eqref{eq: update X rho} regularize the trajectories to reduce the dynamic cost while keeping the interaction and terminal costs unchanged.  
As a result, the three steps together yield a decay in the objective value, as stated in the following two lemmas. 
Proofs are provided in Appendix~\ref{apsec: proof lem descent by X update} and~\ref{apsec: proof lem descent by v update}.
\begin{lemma}
\label{lem: descent by X update}
    Assume that  $F[\rho^{(\ell)}_t],G[\rho^{(\ell)}_1]$ are proper and $L$-smooth in $\bbR^d$ for any $t\in[0,1]$.
    Then for any $0<\alpha_{\ell}<\frac{1}{L}$, the update scheme~\eqref{eq: update X} admits a unique solution $X^{(\ell+1/2)}\in C_X$ and 
    \begin{equation}
        \calJ(X^{(\ell+1/2)};\rho^{(\ell)})\leq \calJ(X^{(\ell)};\rho^{(\ell)})-\frac{1}{2\alpha_{\ell}}\left(\|X^{(\ell+1/2)}-X^{(\ell)}\|_{\mu\otimes [0,1]}^2 + \|X_1^{(\ell+1/2)}-X_1^{(\ell)}\|_{\mu}^2\right).
    \label{eq:decay}
    \end{equation}
\end{lemma}
\begin{lemma}
\label{lem: descent by v update}
    If $X^{(\ell+1/2)}\in C_X$, set $\rho^{(\ell+1)}_t:=(X^{(\ell+1/2)}_t)_{\#}\mu$.
    Then, $\rho^{(\ell+1)}\in C_{\rho}$ and there exists $v^{(\ell+1)}$ solves~\eqref{eq: update v} and $v_t^{(\ell+1)}$ is unique upto a $\rho_t$-zero measure set.
    In addition, if $v^{(\ell+1)}\in C_v$, then there exist $X^{(\ell+1)}\in C_X$ such that $(X^{(\ell+1)},v^{(\ell+1)})\in C_{(X,v)}$, and for all such $X^{(\ell+1)}$, $\calJ(X^{(\ell+1)};\rho)\leq \calJ(X^{(\ell+1/2)};\rho)$.
\end{lemma}

For general MFGs, establishing how the descent property leads to convergence to a fixed point is nontrivial and remains an open question for future investigation.  
In this paper, we show that the proposed scheme achieves a sublinear convergence rate for optimal control problems 
(that is, where $F,G$ are independent of $\rho$) 
and a linear convergence rate when $F,G$ are additionally strongly convex on $\bbR^d$.
The result is stated below,
and the proof is in Appendix~\ref{apsec: proof thm cvg rate}.
\begin{theorem}
\label{thm: cvg rate}
Let the sequence $X^{(\ell)},v^{(\ell)}$ generated by update scheme~\eqref{eq: update X rho}\eqref{eq: update X}\eqref{eq: update v} with $0<\alpha_{\ell}=\alpha<\frac{1}{L}$,
and assume $v^{(\ell)}\in C_v$ for all $\ell$.
If $F,G$ are independent of $\rho$ and are proper and $L$-smooth functions in $\bbR^d$, and $\calJ(X)\geq \underline{\calJ}$ for any $X\in C_X$, then $X^{(\ell)}$ satisfies
    \begin{equation}
        \min_{\ell\leq K}\left\{\|X^{(\ell+1/2)}-X^{(\ell)}\|^2_{\mu\otimes [0,1]} + \|X_1^{(\ell+1/2)}-X_1^{(\ell)}\|^2_{\mu}\right\}\leq\frac{2\alpha(\calJ(X^{(0)})-\underline{\calJ})}{K},
    \label{eq:sublinear}
    \end{equation}
    If in addition, $F$ and $G$ are $\lambda$-convex $(\lambda> 0)$, and there exist a solution $(X^*,\rho^*)$ to~\eqref{eq: mfg X}, then the sequence $X^{(\ell)}$ generated by update scheme~\eqref{eq: update X} and $X^{(\ell)}=X^{(\ell+1/2)}$ with $0<\alpha_{\ell}=\alpha<\frac{1}{L}$ satisfies
    \begin{equation}
        \|X^{(\ell)}-X^*\|^2_{\mu\otimes [0,1]} + \|X^{(\ell)}_1 - X^*_1\|^2_{\mu}\leq\frac{1}{(1+2\lambda\alpha)^{\ell}}\left( \|X^{(0)}-X^*\|^2_{\mu\otimes [0,1]} + \|X^{(0)}_1 - X^*_1\|^2_{\mu}\right).
    \label{eq:linear}
    \end{equation}
\end{theorem}

\section{Algorithm }

We provide the implementation details of the method proposed in Section \ref{sec:method}.
In practice, we parametrize the velocity field $v$ by a neural network $v_{\theta}$.
Since $X(x,\cdot)$ is the trajectory starting at $x$, we represent $X$ with particle trajectories at discrete time stamps.
More specifically, let $\Delta t = \frac{1}{m}$ and the time steps $t_j = j\Delta t$ for $j = 1, \ldots, m$. 
We sample $\{x_i\}_{i=1}^n$ i.i.d. from $\mu$
and the trajectory of the $i$-th particle $x_i$ is $\{X_{i,t_j},j\leq m\}$ where $X_{i,t_j}\approx X(x_i, t_j)$. Then the collection of trajectories $\{X_{i,t_j}\}_{i,j}$ approximates $X$ and the empirical distribution of $\{X_{i,t_j}\}_i$ approximates $\rho_{t_j}$.
With the finite difference and sampling approximation, we detail the implementation as follows.

\RestyleAlgo{ruled}
\SetKwComment{Comment}{/* }{ */}
\SetKwInput{KwPar}{Parameter}
\SetKwInput{KwIni}{Initialization}
\SetKw{KwBreak}{Break}

\begin{algorithm}[t]
\caption{Particle-based Flow-Matching for Mean-Field Games}
\label{alg: gradient descent}
  \KwIn{initial distribution $\mu$, interaction cost $F$, terminal cost $G$ and initial velocity $v_{\theta}$\;}
  \KwPar{total outer loop $K$\; 
  \quad for particle update: particle batch size $n_1$, number of particle updates $L_1$\; 
  \quad for flow matching: particle batch size $n_2$,
  number of velocity network updates $L_2$\;}
  \For{$k=1,2,\cdots,K$}{

    Resample: 
    sample $\{X_{i,t_0}\}_{i=1}^n\sim_{i.i.d}\mu$, and compute trajectories $\{X_{i,t_j}\}_{j=1}^m$ via $v_{\theta}$ for $i=1,\cdots,n$\;

    \For{$l=1,2,\cdots,L$}{
  
    Cost estimation: 
    compute costs $F[\rho_{t_j}], G[\rho_{t_m}]$ and their (sub)gradients (see Exs.~\ref{eg:ot-kl cost},~\ref{eg:nonpot cost})\;

    Particle update: 
    update $\{X_{i,t_j,},1\leq j\leq m\}_{i}$ by~\eqref{eq:gradient descent} with batch size $n_1$ for $L_1$ steps\;
    }
    Flow matching: 
    update parameters $\theta$ of the velocity network $v_{\theta}$ to match empirical trajectories of $\{X_{i,t_j},j\leq m\}_{i=1}^n$ with batch size $n_2$ for $L_2$ steps of Adam with the training loss~\eqref{eq:flow matching}\;
  }
  \KwOut{parameterized velocity field $v_{\theta}$.}
\end{algorithm}

\begin{enumerate}
    \item Resample and cost estimation.
    
    While obtaining $X^{(\ell)}$ by resampling and solving ODE is straightforward, computing $\rho^{(\ell)}$ and therefore the cost can be costly.
    In practice, we use empirical samples $\{X_{i,t_j}\}_i$ of $\rho_{t_j}$ to estimate the coupling costs $F[\rho_{t_j}]$, $G[\rho_{t_m}]: \bbR^d \to \bbR$ and their (sub)gradients. This procedure depends on the specific form of the coupling. We provide two representative cases, which will be numerically studied in Section \ref{sec:numerical}:
    
    \begin{example}[Dynamic OT with relaxed terminal cost]
        \label{eg:ot-kl cost}
        Consider the OT problem with terminal constraint relaxed to KL divergence, i.e. $F=0, G[\rho]=\log\frac{\dd\rho}{\dd\nu}$.
        In practice, samples from both $\rho$ and $\nu$ are available. Following~\citep{xu2025qflow}, we approximate $G[\rho_{t_m}]$ by training a neural network classifier to distinguish between samples from $\rho_{t_m}$ and $\nu$. The (sub)gradient of $G[\rho_{t_m}]$ at a point $x$ can then be computed via auto-differentiation.
        When the cost is updated frequently, the particle distribution changes gradually, so the classifier can be trained for fewer steps at each update. 
    \end{example}

    \begin{example}[Smoothing coupling]
        \label{eg:nonpot cost}
        Consider nonlocal coupling $F[\rho](x) = \int K(x,y) \dd\rho(y)$ where the kernel function $K(x,y)$ is differentiable in $x$, and then $\nabla F[\rho](x) = \int \nabla_x K(x,y) \dd\rho(y)$. Both $F[\rho]$ and $\nabla F[\rho]$ can be approximated via empirical averages over particles.
        Consider terminal cost $G[\rho](x)=(a^\top x-c)^2$ independent of $\rho$, then $\nabla G[\rho]$ has a closed-form.
    \end{example}

    \item Particle update.

    Given $\rho$ and particle trajectories $X$, assume $F[\rho_t]$ and $G[\rho_1]$ are differentiable and that $X$ is sufficiently smooth. Then, for any smooth perturbation $Y: \bbR^d \times [0, T] \to \bbR^d$ with $Y(x,0) = 0$, a first-order expansion gives
    \begin{align*}
        \calJ(X+\epsilon Y;\rho) 
        &=\calJ(X;\rho)
        +\epsilon \left( \int_0^1 \left \langle Y_t,-(\partial_{tt}X)_t+\nabla F[\rho_t](X_t)\right \rangle_{\mu}\dd t
         + \left\langle Y_1,(\partial_tX)_1+\nabla G[\rho_1](X_1) \right\rangle_{\mu} \right) \\
        &~~~ + o(\epsilon),
    \end{align*}
    where the first-order term in $\epsilon$ provides the descent direction. 
    In practice, we update particle trajectories $\{X_{i,t_j},1\leq j\leq m\}_i$ for $L_1$ steps with batch size $n_1$ by:
    \begin{equation}
    X_{i,t_j}\leftarrow
    \left\{\begin{aligned}
        &X_{i,t_j} - \beta\Delta t\left( -(D_{tt}X_i)_{t_j}+\nabla F[\rho_{t_j}](X_{i,t_j}) \right),\quad j=1,2,\cdots,m-1,\\
        &X_{i,t_m} - \beta\left( (D_tX_i)_{t_m}+\nabla G[\rho_{t_m}](X_{i,t_m}) \right),\quad j=m.
    \end{aligned}\right.
    \label{eq:gradient descent}
    \end{equation}
    Here $\leftarrow$ indicates that $X_{i,t_j}$ is updated by the right-hand side expression, $\beta$ is the step size, and $D_{tt}$ and $D_t$ are finite-difference approximations of $\partial_{tt}$ and $\partial_t$, respectively.

    \item Flow matching update.
    
    The flow matching update trains the velocity network $v_{\theta}$ to match the empirical velocities of particle trajectories $\{X_{i,t_j},\, j \leq m\}_{i=1}^n$, which is closely related to flow matching methods~\citep{albergo2023stointerp,lipman2023flowmatching,liu2022rectifiedflow}. The training loss is
    \begin{equation}
        \min_{\theta} \, \frac{\Delta t}{n} \sum_{i=1}^n \sum_{j=1}^{m} 
        \left\| v_{\theta}(X_{i,t_{j-1}}, t_{j-1}) - \frac{1}{\Delta t}
        (X_{i,t_j} - X_{i,t_{j-1}}) \right\|^2.
        \label{eq:flow matching}
    \end{equation}
    We optimize $\theta$ for $L_2$ steps using Adam, with a batch of $n_2$ trajectories per step.
    Since the particle update does not require access to the neural network $v_{\theta}$, and flow matching reduces the cost only when trajectories $X$ intersect, we apply the flow matching update less frequently.
\end{enumerate}

To summarize, at the $\ell$-th epoch, we first sample $n$ particles and compute their trajectories using $v_{\theta}$, then run the particle update for $L$ iterations. Afterward, we update the neural network by matching the computed trajectories. The full procedure is in Algorithm~\ref{alg: gradient descent}.

The main computational costs come from the particle update (line 5) and the flow matching update (line 7). The cost of the particle update scales as $O(n_1 L_1)$, and that of the flow matching update scales as $O(n_2 L_2)$, both proportional to the batch size and number of update steps.
Importantly, the flow matching update is performed infrequently. In addition, compared to simulation-based methods, our approach offers a simpler and more efficient training process: the velocity network is updated via a decoupled, simulation-free flow matching step, which significantly reduces overall computational cost.

\section{Experiments}
\label{sec:numerical}

We apply the proposed method to two types of MFGs:
Section \ref{subsec:num toy} and \ref{subsec:num image} compute the OT interpolation between two distributions via an MFG formulation,
and Section \ref{subsec:num non-pot} considers a non-potential MFG. 
Details of the experimental setup, including hyperparameters and neural network structures, are provided in Appendix~\ref{appendix:experiment}.
Code can be found at \url{https://github.com/Jayaos/mfg_flow_matching}.

\begin{figure}[t]
    \centering
    \includegraphics[width=\textwidth]{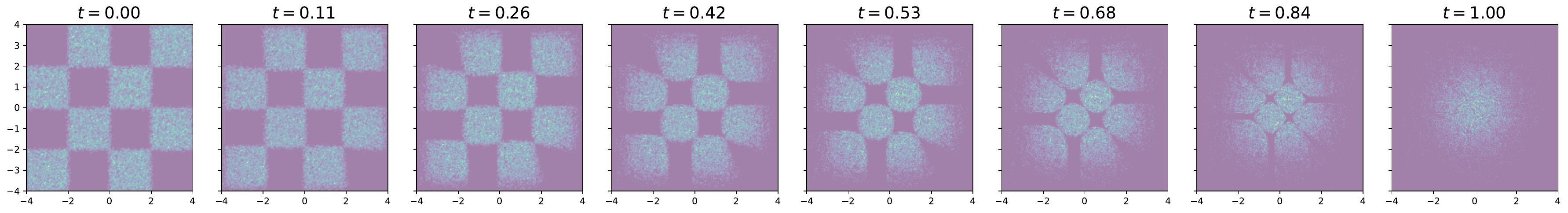}
    \caption{Intermediate distributions between $4 \times 4$ checkerboard and an isotropic Gaussian solved by the proposed algorithm.}
\label{fig:checkerboard-to-gaussian}
\end{figure}

\subsection{2D toy example}
\label{subsec:num toy}

We are to learn the dynamic OT in $\mathbb{R}^2$ between a $4 \times 4$ checkerboard distribution and an isotropic Gaussian target,
and we adopt the KL-divergence relaxation of the terminal constraints as in Example \ref{eg:ot-kl cost}. 
Figure~\ref{fig:checkerboard-to-gaussian} visualizes the evolution of the distribution, where lighter colors indicate more particles and higher density. 
The learned transport map continuously interpolates between the checkerboard source density and the target Gaussian density.

\subsection{Non-potential MFG}
\label{subsec:num non-pot}

\begin{figure}
    \centering
    \begin{subfigure}[t]{0.45\textwidth}
        \centering
        \includegraphics[height=4.2cm]{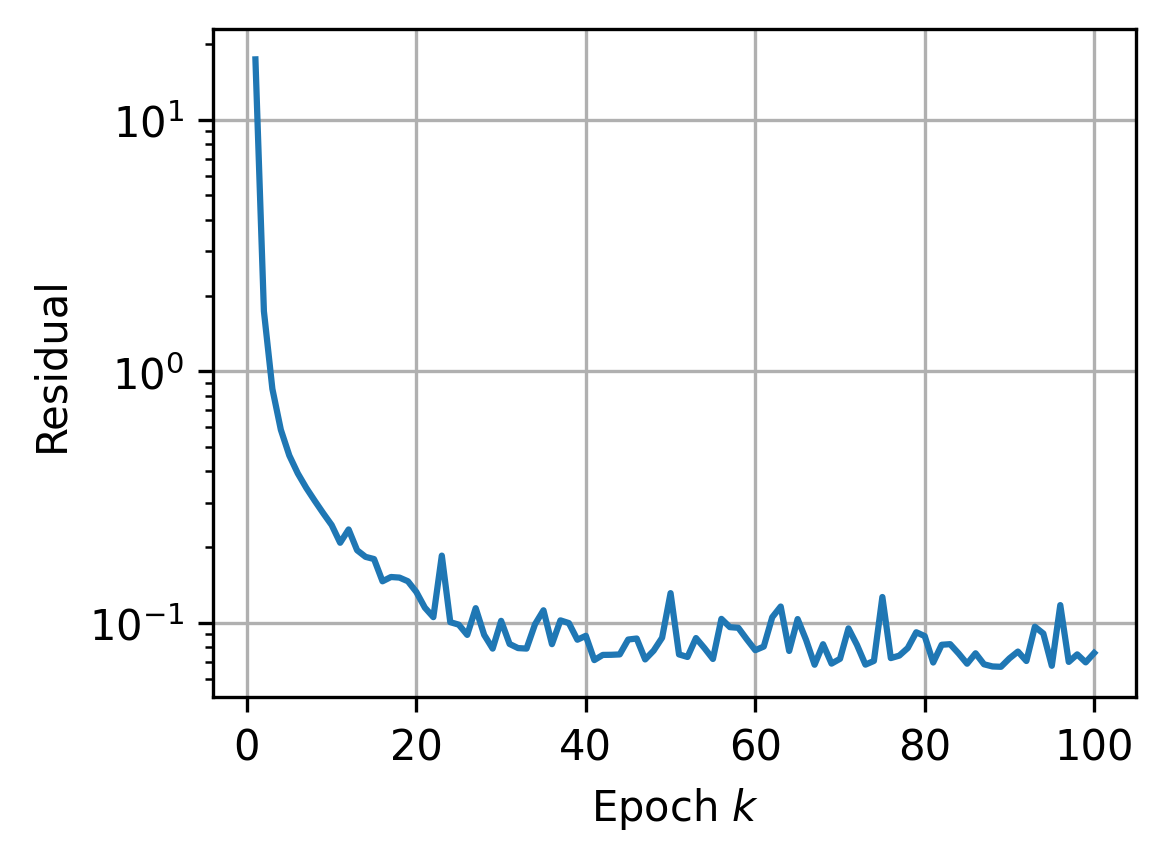}
        \caption{Residual versus the number of training epoch $k$.
        The residual is defined in~\eqref{eq:def residual} and reflects how much the necessary condition of the MFG is violated.} 
        \label{fig:asymker-res}
    \end{subfigure}
    \hfill
    \begin{subfigure}[t]{0.45\textwidth}
        \centering
        \includegraphics[height=4.2cm]{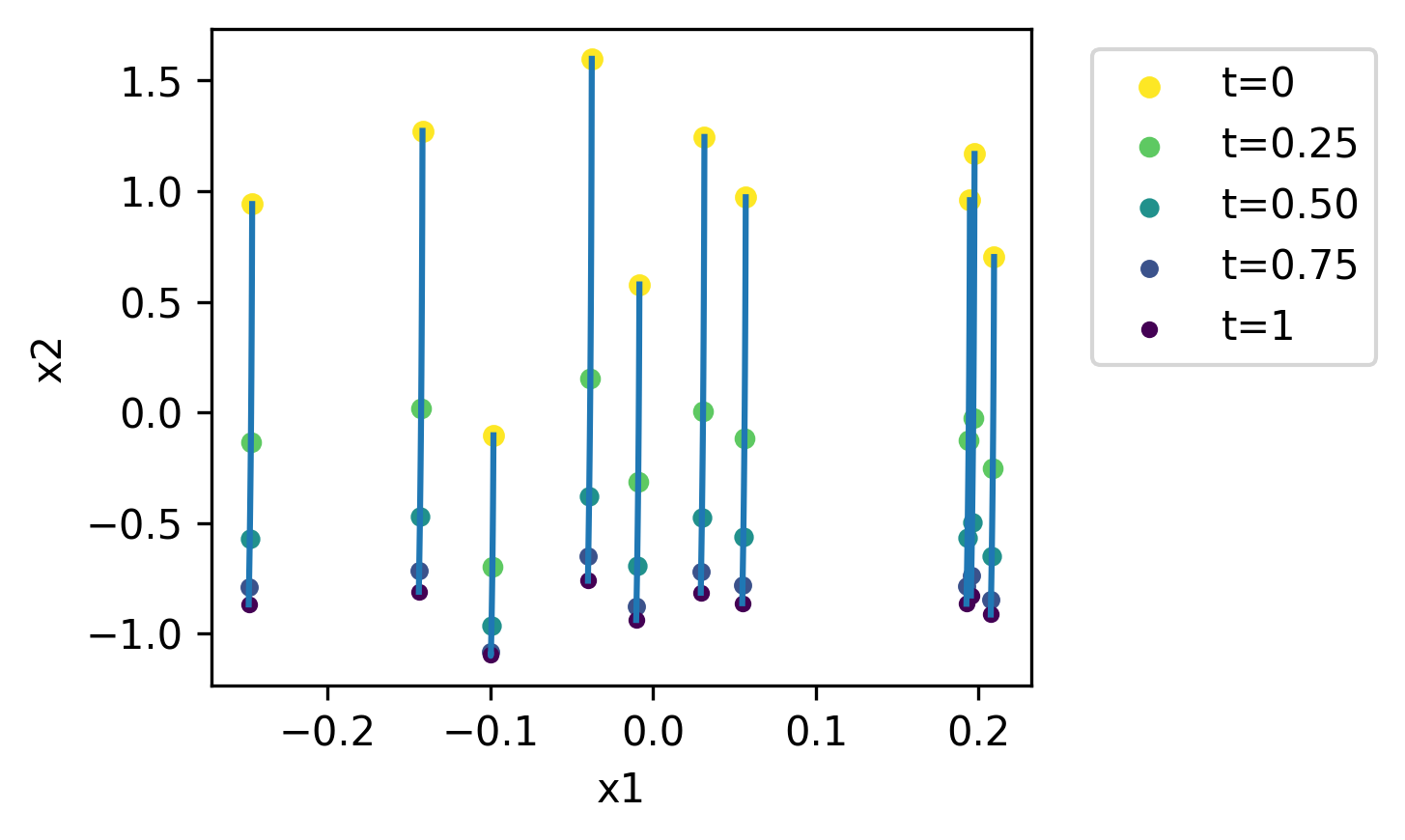}
        \caption{Sample testing trajectories from the learned control $v_\theta$. Players move toward the target line $x_2 = -1$, adapting their speed to reduce the costs.}
        \label{fig:asymker-traj}
    \end{subfigure}
    
    \vspace{0.5em}  
    
    \begin{subfigure}[t]{\textwidth}
        \centering
        \includegraphics[width=\textwidth]{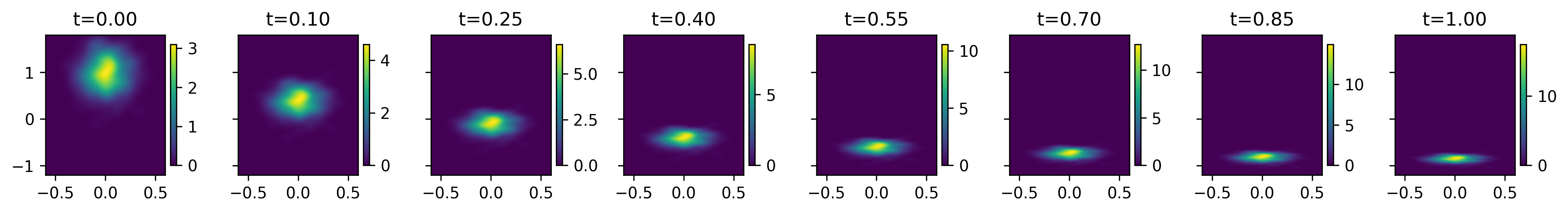}
        \caption{Estimated density evolution from test trajectories. The population tends to concentrate along $x_2$ and spreads along $x_1$ due to asymmetric interaction.}
        \label{fig:asymker-density}
    \end{subfigure}

    \caption{Numerical results for the non-potential MFG in Section~\ref{subsec:num non-pot}.     
    (a) shows convergence of the algorithm.
    (b) and (c) illustrate that the algorithm output aligns with the physical intuition of a Nash equilibrium.}
    \label{fig:non-pot}
\end{figure}

We consider an MFG on $\mathbb{R}^d$ with interaction cost 
$F[\rho](x, t) = \lambda_F \int K(x, y)\dd\rho(y)$, where $K(x, y) = \exp(a^\top(x - y))$ and $a \in \mathbb{R}^d$ is a fixed non-zero vector.
This is the smoothing coupling cost as in Example~\ref{eg:nonpot cost}. 
Since $K$ is asymmetric, the game is non-potential.
The terminal cost is $G(x) = \lambda_G(x_2 - c)^2$, where $c \in \mathbb{R}$ and $x_2$ denotes the second coordinate of $x$, encouraging players to move toward $x_2 = c$.
We set $d = 2$ and the initial distribution is Gaussian with mean $[0,1]^\top$ and covariance $\operatorname{diag}(0.02, 0.1)$. 
The model parameters are $\lambda_F = 10$, $a = [0,1]^\top$, $\lambda_G = 1$, and $c = -1$. 
As discussed in Example~\ref{eg:nonpot cost}, evaluating $F[\rho]$ and its gradient $\nabla F[\rho]$ only requires empirical averages of $e^{a^\top x}$, while $G[\rho]$ and its gradient $\nabla G[\rho]$ admit closed-form evaluations.

Training performance is measured by the residual 
\begin{equation}
\left( 
\frac{\Delta t}{n} \sum_{i=1}^{n} \sum_{j=1}^{m-1} \left\| -(D_{tt} X_i^{(\ell)})_{t_j} + \nabla F[\rho_{t_j}^{(\ell)}](X_{i,t_j}^{(\ell)}) \right\|^2
+ \frac{1}{n} \sum_{i=1}^n \left\| (D_t X_i^{(\ell)})_{t_m} + \nabla G[\rho_{t_m}^{(\ell)}](X_{i,t_m}^{(\ell)}) \right\|^2
\right)^{1/2},
\label{eq:def residual}
\end{equation}
which is a sample-based approximation of the $\mu\otimes[0,1]$ norm of the first-order variation of $\mathcal{J}(X;\rho^{(\ell)})$ at $X=X^{(\ell)}$. 
A small residual indicates proximity to a stationary point. 
Figure~\ref{fig:asymker-res} shows that the algorithm converges to a state with a residual of $10^{-1}$.

To evaluate the learned policy $v_\theta$, we simulate trajectories of testing particles resampled from the initial distribution. 
Figure~\ref{fig:asymker-traj} shows 10 testing trajectories, and Figure~\ref{fig:asymker-density} shows the density evolution of testing particles via Gaussian kernel density estimation. 
Players move toward the terminal line $x_2 = -1$ to reduce terminal cost. 
Those with larger initial $x_2$ move faster to escape higher interaction cost, while those with smaller $x_2$ balance interaction and dynamic costs. 
As players evolve, the population compresses along $x_2$ axis and spreads along $x_1$ axis, making $x-y$ nearly orthogonal to $a$. 
This profile aligns with the optimal strategy of matching population pace while progressing toward $x_2 = -1$.
More detailed discussion is in Appendix~\ref{apsubsec:non-pot}.

\begin{figure}[t]
    \centering
    \begin{subfigure}[t]{0.48\textwidth}
        \centering
        \includegraphics[width=\textwidth]{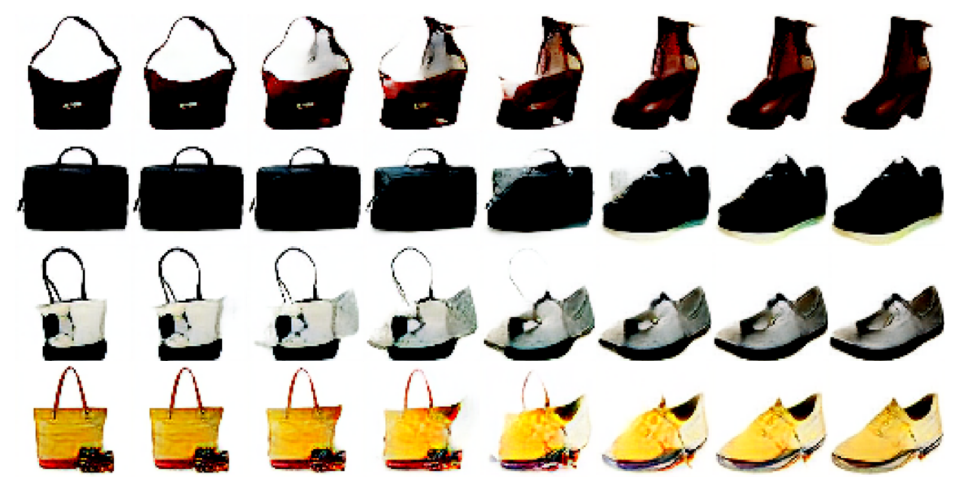}
        \caption{Handbags $\to$ Shoes}
    \end{subfigure}
    \hspace{1em} 
    \begin{subfigure}[t]{0.48\textwidth}
        \centering
        \includegraphics[width=\textwidth]{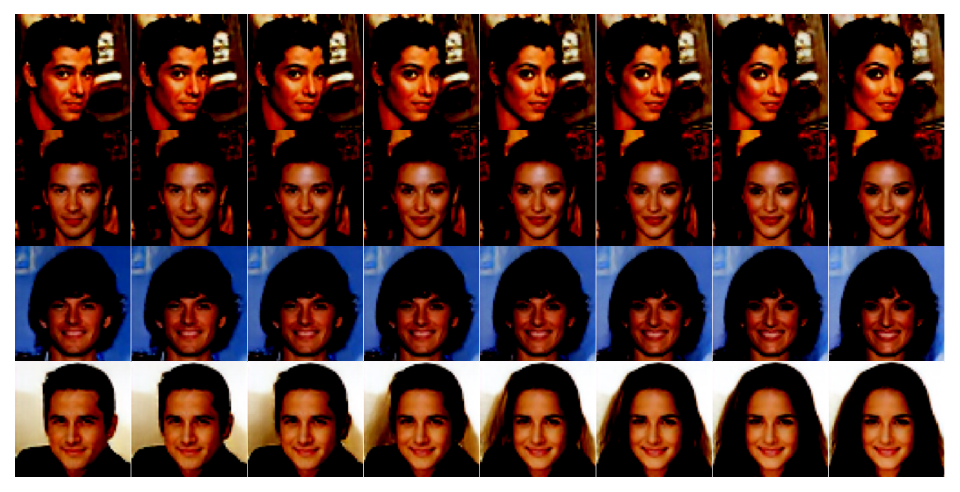}
        \caption{CelebA male $\to$ CelebA female}
    \end{subfigure}
    
    \caption{Randomly selected three OT trajectories for image-to-image translation of (a) handbags $\to$ shoes and (b) CelebA male $\to$ CelebA female.}
    \label{fig:image_translation_trajectory}
\end{figure}

\subsection{Image-to-image translation}
\label{subsec:num image}

We use Algorithm~\ref{alg: gradient descent} to learn the dynamic OT between distributions of two sets of RGB images. The first set contains Shoes~\citep{yu2014fine} and Handbags~\citep{zhu2016generative} and the second set contains CelebA male and female~\citep{liu2015deep} images.
Specifically, the goal is to conditionally generate shoe/CelebA female images by transporting handbag/CelebA male images, respectively. 
We evaluate the quality of translated images by using Fréchet inception distance (FID)~\citep{heusel2017gans} and compare with different baselines: flow matching methods including Optimal Transport Conditional Flow Matching (OT-CFM)~\citep{tong2023improving}, Stochastic Interpolants (SI)~\citep{albergo2023stointerp}, Rectified Flow~\citep{liu2023flow}, Schrödinger Bridge Conditional Flow Matching (SB-CFM)~\citep{tong2023improving}; two representative generative adversarial networks (GANs), Disco GAN~\citep{kim2017learning} and Cycle GAN~\citep{zhu2017unpaired}; Neural Optimal Transport (NOT)~\citep{korotin2023neural}; and Q-flow~\citep{xu2025qflow}, which aims to learn dynamic OT using flow.

We train a deep variational autoencoder (VAE) to compress images into a latent space following~\citet{rombach2022high}, so that learning dynamic OT between two sets of images can be conducted in the latent space for all methods. Table~\ref{tab:fid_comparison} reports FID between translated samples from the test set and ground-truth images from the held-out set, demonstrating that our method is comparable to Q-Flow and outperforms the baselines. Figure~\ref{fig:image_translation_trajectory} visualizes OT trajectories from handbag/CelebA male to shoe/CelebA female images, respectively. We can observe that our method produces smooth and coherent translations, particularly in terms of color consistency and reduction of visual artifacts.

\begin{table}[t]
\centering
\caption{FID of translated images of our method and baselines. Lower FID is better. Results for Disco GAN, Cycle GAN, and NOT are reported from \citet{korotin2023neural}. Results for Q-Flow is reported from~\citet{xu2025qflow}.}
\resizebox{\columnwidth}{!}{
\begin{tabular}{lcccccccccc}
\toprule
& \textbf{Ours} & Q-Flow & OT-CFM & SI & Rectified Flow & SB-CFM & Disco GAN & Cycle GAN & NOT  \\
\midrule
Handbag $\to$ Shoes & 12.44 & 12.34 & 13.01 & 15.87 & 13.91  & 12.70 & 22.42 &  16.00 & 13.77 \\
Male $\to$ Female  & 9.68 &  9.66 &  12.88 & 16.39 & 14.01 & 11.55 & 35.64 & 17.74  & 13.23 \\
\bottomrule
\end{tabular}
}
\label{tab:fid_comparison}
\end{table}

\section{Discussion}

The work can be extended in several future directions.
First, although the particle-based method is simulation-free, it incurs significant memory costs when the number of trajectories and that of time steps are large.
Second, while our framework applies to general MFGs without requiring a variational formulation, the theoretical convergence guarantees currently only cover the optimal control setting.
Third, the convergence analysis considers the idealized proximal fixed point scheme, and it would be useful to develop a more complete analysis that accounts for practical sources of error, such as sampling, finite difference approximations, and flow matching inaccuracies.
Finally, it would be useful to extend the method to more datasets and real applications. 

\subsubsection*{Acknowledgments}
The authors thank Chen Xu and Edward Chen for their help with numerical experiments. 
The project was supported by the Simons Foundation (grant ID: MPS-MODL-00814643).
XC was also partially supported by 
NSF DMS-2237842. The work of JL and YX was partially funded by NSF DMS-2134037, CMMI-2112533, and the Coca-Cola Foundation.

\bibliography{reference}
\bibliographystyle{iclr2026_conference}

\appendix
\section{Table of notations}
\label{apsec:notation}

\renewcommand{\arraystretch}{1.2}

\begin{longtable}{>{\raggedright\arraybackslash}m{2cm} p{12.5cm}}
\caption{Summary of Notations} \label{tab:notation} \\

\toprule
\textbf{Notation} & \textbf{Description} \\
\midrule
\endfirsthead

\toprule
\textbf{Notation} & \textbf{Description} \\
\midrule
\endhead

\bottomrule
\multicolumn{2}{r}{\textit{(continued on next page)}} \\
\endfoot

\bottomrule
\endlastfoot

$\calP_2(\bbR^d)$ & Set of Borel probability measures on $\bbR^d$ with finite second moment. \\
$L^2(\mu)$ & Space of functions $\xi:\bbR^d\to\bbR^d$ such that $\int\|\xi(x)\|^2\,\dd\mu(x)<\infty$. \\
$\langle\xi,\eta\rangle_{\mu}$ & Inner product on $L^2(\mu)$ defined as $\int \langle\xi(x),\eta(x)\rangle\,\dd\mu(x)$. \\
$\|\xi\|_{\mu}$ & Norm on $L^2(\mu)$ defined as $\left(\langle\xi,\xi\rangle_{\mu}\right)^{1/2}$. \\
$L^2(\mu \otimes [0,1])$ & Space of functions $\xi:\bbR^d \times [0,1] \to \bbR^d$ with finite $\mu \otimes \text{Lebesgue}$-norm. \\
$\langle\xi,\eta\rangle_{\mu \otimes [0,1]}$ & Inner product on $L^2(\mu \otimes [0,1])$ defined as $\int_0^1 \langle\xi_t,\eta_t\rangle_{\mu}\,\dd t$. \\
$\|\xi\|_{\mu \otimes [0,1]}$ & Norm on $L^2(\mu \otimes [0,1])$ defined as $\left( \langle\xi,\xi\rangle_{\mu \otimes [0,1]} \right)^{1/2}$. \\
\midrule
$\mu$ & Fixed initial distribution in $\calP_2(\bbR^d)$. \\
$\rho$ & Population distribution. Depending on context, $\rho\in \calP_2(\bbR^d)$ or $\rho:[0,1]\to \calP_2(\bbR^d),\rho_t=\rho(t)$. \\
$v$ & Control function, velocity field. Depending on context, $v:\bbR^d \to \bbR^d$ or $v:[0,1] \to L^2(\bbR^d;\bbR^d),v_t=v(t)$. \\
$X$ & Characteristic map. Depending on context, $X:\bbR^d \to \bbR^d$ or $X:[0,1] \to L^2(\mu),X_t=X(t)$. \\
$m$ & Momentum, vector-valued Radon measure $m=\rho v$.\\
\midrule
$C_\rho$ & Space of absolutely continuous curves $AC^2(0,1;(\calP_2(\bbR^d), W_2))$ with finite quadratic energy. \\
$C_X$ & Set of maps $X:\bbR^d \times [0,1] \to \bbR^d$ such that $X(x,\cdot)\in AC^2(0,1;\bbR^d)$ for $\mu$-a.e.\ $x$. \\
$C_v$ & Set of functions $v:\bbR^d\times[0,1]\to\bbR^d$ that are bounded and Lipschitz continuous in $x$ on every compact set $B\subset\bbR^d$\\
$C_{(\rho,v)}$ & Feasible set of pairs of distribution flow and velocity field $(\rho,v)$, defined in~\eqref{eq: cstr set of rho,v}. \\
$C_{(X,v)}$ & Feasible set of pairs of characteristic map and velocity field $(X,v)$, defined in~\eqref{eq: C_X,v}. \\
$C_{(\rho,m)}$ & Feasible set of pairs of distribution flow and vector-valued Radon measure flow $(\rho,m)$ defined in~\eqref{eq:def cstr rho m}. \\
\midrule
$\mathcal{R}(\rho,v)$& Dynamic cost of $(\rho,v)\in C_{(\rho,v)}$. $\mathcal{R}(\rho,v):=\int_0^1\int\|v_t(x)\|^2\dd\rho_t(x)$.\\
$\mathcal{R}(X)$& Dynamic cost of $X\in C_X$. $\mathcal{R}(X):=\int_0^1\int\|\partial_tX(x,t)\|^2\dd\mu(x)$.\\
$\mathcal{R}(\rho,m)$& Dynamic cost of $(\rho,m)\in C_{(\rho,m)}$. $\mathcal{R}(\rho,m):=\int_0^1\int\|\frac{\dd m_t}{\dd\rho_t}\|^2\dd\rho_t(x)$.\\
$F$ & Interaction cost function, $\calP(\bbR^d)\times\bbR^d\to\bbR,(\rho,x)\mapsto F[\rho](x)$. See Exs.~\ref{eg:ot-kl cost},~\ref{eg:nonpot cost}. \\
$G$ & Terminal cost function, similar to $F$. \\
$\calJ(\rho,v;\tilde{\rho})$ & Individual optimal control cost on $C_{(\rho,v)}$ for given $\tilde{\rho}\in C_{\rho}$, defined in~\eqref{eq: energh J(rho,v;rho)}. \\
$\calJ(X;\tilde{\rho})$ & Individual optimal control cost on $C_X$ for given $\tilde{\rho}\in C_{\rho}$, defined in~\eqref{eq: obj mfg X}. \\
$\calJ(\rho,m;\tilde{\rho})$ & Individual optimal control cost on $C_{(\rho,m)}$ for given $\tilde{\rho}\in C_{\rho}$, defined in~\eqref{eq:def J(rho,m;rho)}. \\
$\calF, \calG$ & Interaction and terminal potentials, $\calP_2(\bbR^d)\to\bbR$. \\
$\calJ(\rho,v)$ & MFC/potential MFG cost on $C_{(\rho,v)}$, defined in~\eqref{eq: mfc rho v}. \\
$\calJ(X)$ & MFC/potential MFG cost on $C_{X}$, defined in~\eqref{eq: obj mfc X}.  \\
$\calJ(\rho,m)$ & MFC/potential MFG cost on $C_{(\rho,m)}$, $\calJ(\rho,m):=\mathcal{R}(\rho,m)+\int\calF(\rho_t)\dd t+\calG(\rho_1)$. \\ 
\midrule
$m,\Delta t, t_j$ & $m$ is number of time steps, $\Delta t=\frac{1}{m}$ and $t_j=j\Delta t$. \\
$\alpha_{\ell}$ & Weight for fictitious play in~\eqref{eq: ficplay} or the proximal step-size in~\eqref{eq: update X}. \\
$\beta$ & Step-size for particle update in~\eqref{eq:gradient descent}. \\
$K$ & Total number of epochs in Algorithm \ref{alg: gradient descent}. \\
$n, n_1, n_2$ & Number of samples per epoch, particle update batch size, and flow matching update batch size, respectively. \\
$L,L_1, L_2$ & Number of particle update iterations and flow neural network update iterations, respectively. \\
$X_{i,t_j}$ & Approximate location of particle initialized at $x_i$ at time $t_j$, i.e., $X_{i,t_j} \approx X(x_i,t_j)$. \\

\end{longtable}

\section{Other preliminaries}

We first recall the concept of absolutely continuous curves as defined in~\citep[Def. 1.1.1]{AGS_book_gradientflows}.
\begin{definition}[Absolutely continuous curves]
\label{def:ac curve}
    Let $(\mathcal{X},d)$ be a complete metric space and let $l:(0,1)\to\mathcal{X}$ be a curve. 
    We say that $l\in AC^2(0,1;\mathcal{X})$ if there exists a function $m\in L^2(0,1)$ such that
    \begin{equation}
        d(l(s),l(t))\leq \int_s^t m(r)\dd r,\quad\forall 0<s\leq t<b.
    \label{eq:1.1.1}
    \end{equation}
\end{definition}
By~\citep[Thm. 1.1.2]{AGS_book_gradientflows}, for any $l\in AC^2(0,1;\mathcal{X})$, the limit
\begin{equation}
    |l'|(t):=\lim_{s\to t}\frac{d(l(s),l(t))}{|s-t|}
\end{equation}
exists for a.e. $t\in(0,1)$.
In addition, $|l'|\in L^2(0,1)$ and $|l'|$ is an admissible integrand of~\eqref{eq:1.1.1} and for any $m$ satisfying~\eqref{eq:1.1.1}, $|l'|(t)\leq m(t)$ for a.e. $t\in(0,1)$. $|l'|$ is called the metric derivative.

The following definition is also called the Lions derivative. The main difference between it and Fr\'{e}chet derivative is that the perturbation measure is required to have zero measure.
\begin{definition}[First variation \citep{cardaliaguet2017ficplay}]
\label{def:derivative}
Let $\calF:\calP_2(\bbR^d)\to\bbR$, we say that $D\calF:\calP_2(\bbR^d)\times\bbR^d\to\bbR$ is the first variation of $\calF$ if for any $\rho,\rho'\in\calP_2(\bbR^d)$,
\begin{equation}
    \lim_{s\to 0} \frac{1}{s}\left( \calF((1-s)\rho+s\rho')-\calF(\rho)\right) = \int D\calF[\rho](x)\dd(\rho'-\rho)(x).
\end{equation}
we say that $D\calF:\calP_2(\bbR^d)\times\bbR^d\to\bbR$ is the first-order subdifferential of $\calF$ if for any $\rho,\rho'\in\calP_2(\bbR^d)$,
\begin{equation}
    \calF((1-s)\rho+s\rho')-\calF(\rho)  \geq s\int D\calF[\rho](x)\dd(\rho'-\rho)(x).
\end{equation}
Since $D\calF$ is defined up to an additive constant, we assume that
\begin{equation}
    \int D\calF[\rho](x)\dd\rho(x)=0, \quad\forall \rho\in\calP_2(\bbR^d).
\end{equation}

\end{definition}

\section{Proofs}

\subsection{Proof of Proposition \ref{prop:MFC--MFG}}
\label{apsec: proof MFC--MFG}

\begin{proof}
Let $m = \rho v$ be the vector-valued Radon measure, and let $\frac{\dd m}{\dd \rho}$ denote its Radon--Nikodym derivative.  
Define
\begin{equation}
C_{(\rho,m)} := \bigl\{ (\tilde{\rho},\tilde{m}) : \partial_t \tilde{\rho} + \nabla \cdot \tilde{m} = 0, \, \tilde{\rho}_0 = \mu, \, \tilde{\rho} \in C_\rho \bigr\}.  
\label{eq:def cstr rho m}
\end{equation}
The constraint is linear in $(\tilde{\rho},\tilde{m})$, hence the set $C_{(\rho,m)}$ is convex.  

Define the functional
\begin{equation}
\mathcal{R}(\tilde{\rho},\tilde{m}) = 
\begin{cases}
\displaystyle \frac{1}{2}\int_0^1 \int \left\| \frac{\dd \tilde{m}_t}{\dd \tilde{\rho}_t}(x) \right\|^2 \dd \tilde{\rho}_t(x)\,\dd t, & \tilde{m}_t \ll \tilde{\rho}_t \ \forall t\in[0,1], \\
\infty, & \text{otherwise}.
\end{cases}
\end{equation}
Then $\mathcal{R}(\tilde{\rho},\tilde{m})$ is convex in $(\tilde{\rho},\tilde{m})$.  
Next, set
\begin{equation}
\mathcal{J}(\tilde{\rho},\tilde{m};\rho) := \mathcal{R}(\tilde{\rho},\tilde{m}) + \int_0^1 \int F[\rho_t] \,\dd \tilde{\rho}_t(x)\,\dd t + \int G[\rho_1] \,\dd \tilde{\rho}_1(x). 
\label{eq:def J(rho,m;rho)}
\end{equation}
The functional $\mathcal{J}(\tilde{\rho},\tilde{m};\rho)$ is convex in $(\tilde{\rho},\tilde{m})$ under the $L^2$ metric.  

Since the objective is convex and the constraint is linear, $(\tilde{\rho},\tilde{m})$ solves
\begin{equation}
\min_{(\tilde{\rho},\tilde{m}) \in C_{(\rho,m)}} \mathcal{J}(\tilde{\rho},\tilde{m};\rho)    
\end{equation}
if and only if it satisfies the first-order optimality condition:
\begin{equation}
\begin{cases}
-\partial_t \tilde{\phi} + \tfrac{1}{2} \|\nabla \tilde{\phi}\|^2 = F[\rho], & \tilde{\phi}_1 = G[\rho_1], \\
\tilde{m} = -\tilde{\rho} \nabla \tilde{\phi}, \\
\partial_t \tilde{\rho} + \nabla \cdot \tilde{m} = 0, & \tilde{\rho}_0 = \mu.
\end{cases}
\end{equation}
Here, the Hamilton-Jacobi-Bellman equation is understood to hold on the support of $\tilde{\rho}$, with equality replaced by inequality $-\partial_t \tilde{\phi} + \tfrac{1}{2} \|\nabla \tilde{\phi}\|^2 \leq F[\rho],  \tilde{\phi}_1 \leq G[\rho_1]$ outside the support. The Fokker--Planck equation is understood in the sense of distributions.  

Therefore, $(\hat{\rho},\hat{v})$ solves~\eqref{eq: mfg rho v} if and only if it satisfies
\begin{equation}
\begin{cases}
-\partial_t \hat{\phi} + \tfrac{1}{2} \|\nabla \hat{\phi}\|^2 = F[\hat{\rho}], & \hat{\phi}_1 = G[\hat{\rho}_1], \\
\hat{v} = -\nabla \hat{\phi}, \\
\partial_t \hat{\rho} + \nabla \cdot (\hat{\rho} \hat{v}) = 0, & \hat{\rho}_0 = \mu,
\end{cases}
\label{eq:optcond}
\end{equation}

Similarly, deriving the first-order optimality condition for~\eqref{eq: mfc rho v} yields the same system~\eqref{eq:optcond}. Hence, $(\hat{\rho},\hat{v})$ solving~\eqref{eq: mfg rho v} is necessary for solving~\eqref{eq: mfc rho v}, and becomes sufficient provided that $\mathcal{F}, \mathcal{G}$ are convex in $\rho$ under the $L^2$ metric.
\end{proof}

\subsection{Proof of Lemma~\ref{lem: X to rho v}}
\label{apsec: proof lem X to rho v}

We prove the Lemma in two steps.
We first show that the distribution flow $\tilde{\rho}$ induced by $\tilde{X}$ is in $C_{\rho}$ and the interaction and terminal cost are unchanged with the change from Lagrangian coordinate $\tilde{X}$ to Eulerian coordinate $\tilde{\rho}$.
As summarized in the following lemma.
\begin{lemma}
\label{lem: X to rho v, prepare}
    Let $\tilde{X}\in C_X$ and set $\tilde{\rho}_t:=(\tilde{X}_t)_{\#}\mu$. Then $\tilde{\rho}\in C_\rho$, $\int F(x)\dd\tilde{\rho}_t(x)=\mathbb{E}_{x\sim\mu}[F(\tilde{X}_t(x))]$ for any $F:\bbR^d\to\bbR^d$ and $\calF(\tilde{\rho}_t)=\calF((\tilde{X}_t)_{\#}\mu)$ for any $\calF:\calP_2(\bbR^d)\to\bbR$.
\end{lemma}

\begin{proof}
The costs are unchanged directly by the definition of pushforward.

We first prove that $\tilde{\rho}_t\in \calP_2(\bbR^d)$.
    Fix $t\in[0,1]$. For $\mu$-a.e.\ $x$, since $\tilde{X}(x,\cdot)$ is absolutely continuous, we have
    \begin{equation}
        \tilde{X}(x,t) = x + \int_0^t \partial_s \tilde{X}(x,s)\dd s.
    \label{eq1}
    \end{equation}
    Therefore
    \begin{equation}
        \|\tilde{X}(x,t)\|^2\leq 2\|x\|^2 + 2\left\|\int_0^t\partial_s\tilde{X}(x,s)\dd s \right\|^2 
        \leq 2\|x\|^2 + 2\int_0^1 \|\partial_s \tilde{X}(x,s)\|^2\dd s,
    \label{eq2}
    \end{equation}
    where the second inequality is by the Cauchy-Schwarz inequality.
    Since $\mu\in\calP_2(\bbR^d)$ and $\tilde{X}(x,\cdot)\in AC^2(0,1;\bbR^d)$ for $\mu$-a.e.\ $x$, integrating over $\mu$ gives
    \begin{equation}
        \int\|x\|^2\dd\rho_t(x)=\int \|\tilde{X}(x,t)\|^2\dd\mu(x)\leq 2\int\|x\|^2\dd\mu(x) + 2\int\int_0^1\|\partial_s\tilde{X}(x,s)\|^2\dd s\dd\mu(x)<\infty,
    \label{eq3}
    \end{equation}
    therefore $\tilde{\rho}_t\in\calP_2(\bbR^d)$.

For any $0\leq s<t\leq 1$, 
    consider the coupling $\pi=(\tilde{X}_s,\tilde{X}_t)_{\#}\mu$ between $\tilde{\rho}_s$ and $\tilde{\rho}_t$. 
    By definition, we have
    \begin{equation}
        W_2^2(\tilde{\rho}_s,\tilde{\rho}_t)
        \leq \int_{\bbR^d\times\bbR^d}\|y-z\|^2\dd\pi(y,z)
        =\int_{\bbR^d} \|\tilde{X}_s(x)-\tilde{X}_t(x)\|^2\dd\mu(x)
        =\int_{\bbR^d} \left\|\int_s^t\partial_rX(x,r)\dd r\right\|^2\dd\mu(x).
    \label{eq4}
    \end{equation}
Let
    \begin{equation}
        g(r):=\left\|\partial_rX(\cdot,r)\right\|_{\mu}
        =\left(        \int_{\bbR^d}\|\partial_rX(x,r)\|^2\dd\mu(x) \right)^{1/2}.
    \label{eq:g(r)}
    \end{equation}    
Then $g\in L^2([0,1])$ and
    \begin{equation}
        W_2(\tilde{\rho}_s,\tilde{\rho}_t) 
        \leq \left\| \int_s^t\partial_r X(\cdot,r)\dd r \right\|_{\mu}
        \leq \int_s^t \left\|\partial_rX(\cdot,r)\right\|_{\mu}\dd r
        =\int_s^tg(r)\dd r.
    \end{equation}
Therefore $\tilde{\rho}\in C_{\rho}$.
\end{proof}

We then construct the velocity field $\tilde{v}$ such that $(\tilde{\rho},\tilde{v})\in C_{(\rho,v)}$, show that $\tilde{v}$ is the unique solution to the flow matching, and the dynamic cost in $(\tilde{\rho},\tilde{v})$ is bounded above by the dynamic cost in $\tilde{X}$.

\begin{proof}[Proof of Lemma~\ref{lem: X to rho v}]
    We construct $\tilde{v}$ and show that it solves the flow-matching.
    Define the vector measure
    \begin{equation}
        \tilde{m}_t(B):=\int_{\{x:\tilde{X}(x,t)\in B\}}\partial_t\tilde{X}(x,t)\dd\mu(x),\quad B\subset\bbR^d.
    \end{equation}
    Then $\tilde{m}_t\ll\tilde{\rho}_t$ and by Radon-Nikodym theorem, there exist a Borel vector field $\tilde{v}_t$ defined $\rho_t$-a.e. such that $\tilde{m}_t(B)=\int_B\tilde{v}_t(y)\dd\tilde{\rho}_t(y)$.
    This implies that,
    \begin{equation}
        \int \tilde{v}(\tilde{X}(x,t),t) - \partial_t\tilde{X}(x,t)\dd\mu(x)\dd t = 0,
    \label{eq:FM opt cond}
    \end{equation}
    Consider the flow matching problem
    \begin{equation}
        \min_{v} \int_0^1\int\|v(\tilde X(x,t),t)-\partial_t \tilde X(x,t)\|^2\dd\mu(x)\dd t.
    \label{eq:FM X}
    \end{equation}
    \eqref{eq:FM opt cond} shows that $\tilde{v}$ satisfies the optimality condition of~\eqref{eq:FM X}.
    Since the problem~\eqref{eq:FM X} is unconstrained with a convex objective, $\tilde{v}$ is therefore the solution to~\eqref{eq:FM X}.
    
    We then prove that $(\tilde{\rho},\tilde{v})$ satisfies the continuity equation.
    By the chain rule, 
    \begin{equation}
        \frac{\dd}{\dd t}\int\varphi(y)\dd\tilde{\rho}_t(y)
        = \frac{\dd}{\dd t}\int \varphi(\tilde{X}(x,t))\dd\mu(x) 
        = \int \nabla\varphi(\tilde{X}(x,t))\cdot\partial_t\tilde{X}(x,t)\dd\mu(x).
    \label{eq6}
    \end{equation}    
    By the definition of $\tilde{v}$, for any $\varphi\in C_c^{\infty}(\bbR^d)$,
    \begin{equation}
        \int \nabla\varphi(y)\cdot\tilde{v}(y,t)\dd\tilde{\rho}_t(y) 
        = \int \nabla\varphi(\tilde{X}(x,t))\cdot\partial_t\tilde{X}(x,t)\dd\mu(x).
    \label{eq5}
    \end{equation}
    Combining ~\eqref{eq6} and~\eqref{eq5} gives the desired weak formula of the continuity equation:
    \begin{equation}
        \frac{\dd}{\dd t}\int \varphi\dd\tilde{\rho}_t 
        = \int \nabla\varphi\cdot v_t\dd\tilde{\rho}_t.
    \end{equation}
    
    In the end, we show
    \begin{equation}
        \|\tilde{v}_t\|_{\tilde{\rho}_t}^2\leq \|\partial_t\tilde{X}(\cdot,t)\|_{\mu}^2.
    \label{eq7}
    \end{equation}
    This implies $\tilde{v}_t\in L^2(\tilde{\rho}_t)$ and therefore concludes $(\tilde{\rho},\tilde{v})\in C_{(\rho,v)}$.
    In addition, integrating~\eqref{eq7} over time and combining with Lemma~\ref{lem: X to rho v, prepare} gives $\calJ(\tilde\rho,\tilde v;\rho)\leq\calJ(\tilde X;\rho)$.
    
    Define a linear funational $T$ on $L^2(\tilde{\rho}_t)$ by
    \begin{equation}
        T(\psi):=\int \psi(y)\cdot\dd \tilde{m}_t(y) = \int \psi(\tilde{X}(x,t))\cdot\partial_t\tilde{X}(x,t)\dd\mu(x).
    \label{eq:def T}
    \end{equation}
    Then by Cauchu-Schwarz inequality in $L^2(\mu)$
    \begin{equation}
        |T(\psi)|
        \leq \|\psi\circ \tilde{X}_t\|_{\mu} \|\partial_t\tilde{X}(\cdot,t)\|_{\mu} 
        = \|\psi\|_{\tilde{\rho}_t}\|\partial_t\tilde{X}(\cdot,t)\|_{\mu},
    \end{equation}
    which implies that $T$ is a bounded linear functional on the Hilbert space $L^2(\tilde{\rho}_t)$ with operator norm
    \begin{equation}
        \|T\|\leq \|\partial_t\tilde{X}(\cdot,t)\|_{\mu}.
    \end{equation}
    By the Riesz representation theorem, there exists $v_t\in L^2(\tilde{\rho}_t)$ such that for every $\psi\in L^2(\tilde{\rho}_t)$,
    \begin{equation}
        T(\psi)=\int \psi(y)\cdot {v}_t(y)\dd\tilde{\rho}_t(y).
    \label{eq: riesz T}
    \end{equation}
    Moreover,
    \begin{equation}
        \|{v}_t\|^2_{\tilde{\rho}_t} = \|T\|^2 \leq \|\partial_t\tilde{X}(\cdot,t)\|^2_{\mu}.
    \end{equation}
    It suffices to show that $v_t=\tilde{v}_t$.
    Let $B\subset\bbR^d$ be a Borel set and $a\in\bbR^d$ be a fixed vector, then $L^2(\tilde{\rho}_t)\ni\psi(y):=a\mathbf{1}_B(y)=\begin{cases}
        a,& y\in B,\\ 0, & y\not\in B.
    \end{cases}$
    By definition~\eqref{eq:def T}, $T(\psi) = a\cdot \tilde{m}_t(B)$.
    On the other hand, by Riesz representation~\eqref{eq: riesz T}, $T(\psi) = a\cdot\int_B v_t(y)\dd\tilde{\rho}_t(y)$.
    Therefore $a\cdot \tilde{m}_t(B)=a\cdot\int_B v_t(y)\dd\tilde{\rho}_t(y)$ for any $a\in\bbR^d$ and any Borel $B\subset\bbR^d$.
    This implies that $v_t=\frac{\dd \tilde{m}_t}{\dd \tilde{\rho}_t}$ and $v_t=\tilde{v}_t$ except on a $\tilde{\rho}_t$-zero measure set, which concludes the proof. 
\end{proof}

\subsection{Proof of Lemma~\ref{lem: rho v to X}}
\label{apsec: proof lem rho v to X}

This is a direct corollary of~\citep[Prop. 8.1.8]{AGS_book_gradientflows}.

\subsection{Proof of Theorem~\ref{thm: euler=lagrangian}}
\label{apsec: proof thm euler=lagrange}

\begin{proof}
    We first prove (i).
    Since $(\rho^*,v^*)$ solves~\eqref{eq: mfg rho v}, $v^*$ is in $C_v$ and by Lemma~\ref{lem: rho v to X}, there exists $X^*\in C_X$ such that $(X^*_t)_\#\mu=\rho^*_t$ and $\calJ(X^*;\rho^*) = \calJ(\rho^*,v^*;\rho^*)$.
    For any other $\tilde{X}\in C_X$, by Lemma~\ref{lem: X to rho v}, there exist $(\tilde{\rho},\tilde{v})\in C_{(\rho,v)}$ such that $\calJ(\tilde{\rho},\tilde{v};\rho^*)\leq \calJ(\tilde{X};\rho^*)$.
    Therefore, $X^*$ solves~\eqref{eq: mfg X}.

    The same logic applies to MFC and concludes the first part of (ii).

    For the second part of (ii), it suffices to prove $\calJ(\rho^*,v^*)=\calJ(\rho,v)$.
    By optimality of $(\rho,v)$, we have $\calJ(\rho^*,v^*)\geq\calJ(\rho,v)$.
    By Lemma~\eqref{lem: rho v to X}, there is an $X\in C_X$ such that $\calJ(X)=\calJ(\rho,v)$.
    By Lemma~\eqref{lem: X to rho v} and optimality of $X^*$, we have $\calJ(\rho^*,v^*)\leq \calJ(X^*) \leq \calJ(X) = \calJ(\rho,v)$, which concludes the proof.
\end{proof}

\subsection{Proof of Lemma~\ref{lem: descent by X update}}
\label{apsec: proof lem descent by X update}

\begin{proof}
    The dynamic cost
    \begin{equation}
        \mathbb{E}_{x\sim\mu}\left[\int_0^1\frac{1}{2}\|\partial_t\tilde{X}(x,t)\|^2\dd t\right] 
        = \|\partial_t\tilde{X}\|_{\mu\otimes[0,1]}^2
    \end{equation}
    is strictly convex in $\tilde{X}$, 
    In addition, $F[\rho_t^{(\ell)}],G[\rho_1^{(\ell)}]$ are $L$-smooth and $\alpha_{\ell}<\frac{1}{L}$, therefore the objective of the update~\eqref{eq: update X} is strongly convex. Therefore, the minimizer is unique if it exists.

    Let $\{X^{(n)}\}\subset C_X$ be a minimizing sequence $\lim_{n\to\infty}\calJ(X^{(n)};\rho^{(\ell)})=\inf_{X\in C_X}\calJ(X;\rho^{(\ell)})$. 
    Then by $L$-smoothness of $F[\rho_t^{(\ell)}]$ and $G[\rho_1^{(\ell)}]$, and $0<\alpha_{\ell}<\frac{1}{L}$, we have that $$\sup_n \left(\|\partial_t X^{(n)}\|_{\mu\otimes[0,1]} + \|X^{(n)}\|_{\mu\otimes[0,1]} 
    + \|X^{(n)}_1\|_{\mu}\right) < \infty,$$
    which implies $\{X^{(n)}\}$ is bounded in $H^1(0,1;L^2(\mu))$.
    Therefore there exists a subsequence, still denoted as $X^{(n)}$, and $X^*\in H^1(0,1;L^2(\mu))$ such that
    $X^{(n)}\to X^*$ weakly in $H^1(0,1;L^2(\mu))$.

    By weakly convergence in $H^1(0,1;L^2(\mu))$ and weak lower semicontinuity of $L^2$ norm, $\|\partial_tX^*\|_{\mu\otimes[0,1]}^2\leq\liminf_{n\to\infty} \|\partial_t X^{(n)}\|_{\mu\otimes[0,1]}^2$.

    By compact embedding theorem, there exists a subsequence, still denoted as $X^{(n)}$ such that $X^{(n)}\to X^*$ strongly in $C(0,T;L^2(\mu))$ and therefore $X^{(n)}\to X^*$ strongly in $L^2(0,T;L^2(\mu))$ and $X^{(n)}_1\to X^*_1$ strongly in $L^2(\mu)$. 
    Combining with the kinetic term, we have $\calJ(X^*;\rho^{(\ell)})\leq\liminf_{n\to\infty}\calJ(X^{(n)};\rho^{(\ell)})$, which implies $X^*$ is the minimizer.

    By the update rule, $X^{(\ell+1/2)}=X^*$ and the decay property~\eqref{eq:decay} holds.
\end{proof}

\subsection{Proof of Lemma~\ref{lem: descent by v update}}
\label{apsec: proof lem descent by v update}
\begin{proof}
    By Lemma~\ref{lem: X to rho v}, $\rho^{(\ell+1)}\in C_{\rho}$, $v^{(\ell+1)}_t$ is unique upto a $\rho_t$-zero measure set.
    Then by Lemma~\ref{lem: X to rho v} and Lemma~\ref{lem: rho v to X}, we have
    \begin{equation}
        \calJ(X^{(\ell+1)};\rho) =\calJ(\rho^{(\ell+1)},v^{(\ell+1)};\rho) \leq \calJ(X^{(\ell+1/2)};\rho),
    \end{equation}
    which concludes the proof.
\end{proof}

\subsection{Proof of Theorem~\ref{thm: cvg rate}}
\label{apsec: proof thm cvg rate}

\begin{proof}
    Since $F,G$ are independent of $\rho$, combining Lemmas~\ref{lem: descent by X update} and~\ref{lem: descent by v update} with $\alpha_{\ell}=\alpha$ gives
    \begin{equation}
        \frac{1}{2\alpha}\left( \|X^{(\ell+1/2)}-X^{(\ell)}\|_{\mu\otimes[0,1]}^2 + \|X^{(\ell+1/2)}_1-X^{(\ell)}_1\|_{\mu}^2 \right)
        \leq \calJ(X^{(\ell)})
         - \calJ(X^{(\ell+1)})
    \end{equation}
    Telescoping both sides and $\calJ(X)\geq \underline{\calJ}$ gives~\eqref{eq:sublinear}.
    
    If $F,G$ are $\lambda$-convex, then for any $X,Y\in C_X$,
    \begin{equation}
    \begin{aligned}
        &\lambda \left( \|X-Y\|_{\mu\otimes[0,1]}^2 + \|X_1-Y_1\|_{\mu}^2 \right)\\
        &\leq \|\partial_t(X-Y)\|_{\mu\otimes[0,1]}^2
        + \int_0^1\langle \nabla F(X_t)-\nabla F(Y_t),X_t-Y_t\rangle_{\mu}\dd t
        + \langle \nabla G(X_1)-\nabla G(Y_1),X_1-Y_1\rangle_{\mu}
    \end{aligned}
    \end{equation}
    Take $X=X^{(\ell+1)}$ and $Y=X^*$.
    By the update rule and the optimality of $X^*$, we have
    \begin{equation}
    \begin{aligned}
        &\lambda \left( \|X^{(\ell+1)}-X^*\|_{\mu\otimes[0,1]}^2 + \|X_1^{(\ell+1)}-X_1^*\|_{\mu}^2 \right)\\
        &\leq -\frac{1}{\alpha}\left( \langle X^{(\ell+1)}-X^{(\ell)},X^{(\ell+1)}-X^*\rangle_{\mu\otimes[0,1]} 
        + \langle X_1^{(\ell+1)}-X_1^{(\ell)},X_1^{(\ell+1)}-X_1^*\rangle_{\mu} \right),
    \label{eq:monotone}
    \end{aligned}
    \end{equation}
    By algebra on the inner product
    \begin{equation}
        \langle a-b,a-c\rangle = \frac{1}{2}\left( \|a-b\|^2 + \|a-c\|^2 - \|b-c\|^2 \right),
    \end{equation}
    Take $a=X^{(\ell+1)},b=X^{(\ell)},c=X^*$. Then, ~\eqref{eq:monotone} gives
    \begin{equation}
    \begin{aligned}
        &(2\lambda\alpha+1) \left( \|X^{(\ell+1)}-X^*\|_{\mu\otimes[0,1]}^2 + \|X_1^{(\ell+1)}-X_1^*\|_{\mu}^2 \right)
        +\left(\|X^{(\ell+1)}-X^{(\ell)}\|_{\mu\otimes[0,1]}^2 +\|X^{(\ell+1)}_1-X^{(\ell)}_1\|_{\mu}^2\right)\\
        &\leq \left( \|X^{(\ell)}-X^*\|_{\mu\otimes[0,1]}^2 + \|X_1^{(\ell)}-X_1^*\|_{\mu}^2 \right).
    \end{aligned}
    \end{equation}
    Therefore,
    \begin{equation}
        \left( \|X^{(\ell+1)}-X^*\|_{\mu\otimes[0,1]}^2 + \|X_1^{(\ell+1)}-X_1^*\|_{\mu}^2 \right)
        \leq \frac{1}{1+2\lambda\alpha}\left( \|X^{(\ell)}-X^*\|_{\mu\otimes[0,1]}^2 + \|X_1^{(\ell)}-X_1^*\|_{\mu}^2 \right),
    \end{equation}
    and~\eqref{eq:linear} holds.
\end{proof}

\section{Experiment details}
\label{appendix:experiment}

\subsection{2D toy example}

We use multilayer perceptrons (MLP) with six hidden layers of width 512 and Rectified Linear Unit (ReLU) activations to model the velocity field and the classifier. The time $t$ is concatenated with the input, and the concatenated input is fed into the velocity field. We use \texttt{rk4} to solve ODEs. All other hyperparameters, aside from the model architecture, are listed in Table~\ref{tab:toy_example_hyperparams}.

\begin{table}[h]
\centering
\caption{Hyperparameters for 2D toy example.}
\begin{tabular}{l c}
\hline
\textbf{Hyperparameter} & \textbf{Value} \\
\hline
Total outer loop                     & 20    \\
Classifier training batch size        & 2048  \\
Classifier training steps             & 1000  \\
Classifier learning rate              & 0.001 \\
Classifier intermediate training frequency & 10    \\
Classifier intermediate training steps & 20    \\
Velocity field training batch size    & 2048  \\
Velocity field training steps         & 1000  \\
Velocity field learning rate          & 0.001 \\
Particle optimization batch size      & 2048  \\
Particle optimization steps           & 1000  \\
Particle optimization learning rate   & 0.001 \\
Number of timesteps                   & 10    \\
\hline
\end{tabular}
\label{tab:toy_example_hyperparams}
\end{table}

\subsection{Non-potential MFG}
\label{apsubsec:non-pot}

Notice that in~\eqref{eq: mfg X}, a vanishing first-order variation is a necessary condition for $(X, \rho)$ to be a fixed point. 
The residual defined in~\eqref{eq:def residual} provides a sampling-based approximation of the $L^2(\mu)$ norm of the first-order variation of the functional $\mathcal{J}(X; \rho^{(\ell)})$ evaluated at $X=X^{(\ell)}$, i.e.
\begin{equation}
    \left( \int_0^1\left\|-(\partial_{tt}X^{(\ell)})_t+\nabla F[\rho_t^{(\ell)}]\right\|_{\mu}^2\dd t 
    + \left\|(\partial_tX^{(\ell)})_1 + \nabla G[\rho_1^{(\ell)}](X_1^{(\ell)})\right\|_{\mu}^2 \right)^{1/2}.
\label{eq:def residual cts}
\end{equation}
Therefore, a small residual~\eqref{eq:def residual} is necessary for $(X^{(\ell)},\rho^{(\ell)})$ to approximate a fixed point. Figure~\ref{fig:asymker-res} illustrates that our algorithm achieves a residual on the order of $10^{-1}$.

To assess whether the learned neural network control $v_{\theta}$ induces a Nash equilibrium, we resample testing particles from the initial distribution and simulate their trajectories using $v_\theta$. Figure~\ref{fig:asymker-traj} shows 10 testing trajectories. Figure~\ref{fig:asymker-density} displays the density evolution based on all testing trajectories, estimated using Gaussian kernel density estimation.

All players move toward the line $x_2 = -1$ to reduce the terminal cost. For players with initially larger $x_2$ values, the interaction cost dominates early in the trajectory since there is more mass ahead of them along $a$. To mitigate this, they move quickly in the direction of $a$. Players starting at smaller $x_2$ values also tend to move in the direction of $a$ to avoid falling behind when other players catch up. However, because their initial interaction cost is lower, they balance interaction and dynamic costs, leading to slower initial movement.
Although the initial distribution has a larger variance along the $x_2$ direction, players' desire to reduce interaction cost causes them to compress along $x_2$ and spread out along $x_1$ after $t = 0.25$, making $x-y$ nearly orthogonal to $a$. Under this population profile, the optimal strategy becomes matching the pace of the group while moving toward the target terminal location $x_2 = -1$.

Overall, the trajectories in Figure~\ref{fig:asymker-traj} and the density evolution in Figure~\ref{fig:asymker-density} are consistent with the physical interpretation of a Nash equilibrium for this test case.

The hyperparameters for the algorithm are detailed below. 
We use MLP with 3 hidden layers of width [4, 8, 16]. We use \texttt{euler} (forward Euler) to solve ODEs. All other hyperparameters, aside from the model architecture and pre-specified constants, are listed in Table~\ref{tab:np_mfg_hyperparams}.

\begin{table}[h]
\centering
\caption{Hyperparameters for Non-Potential MFG.}
\begin{tabular}{l c}
\hline
\textbf{Hyperparameter} & \textbf{Values} \\
\hline
Total outer loop                       & 100    \\

Velocity field training steps         & 100  \\
Velocity field learning rate          & 0.01 \\

Particle optimization steps           & 100  \\
Particle optimization learning rate   & 0.01 \\
Number of timesteps                   & 20    \\
\hline
\end{tabular}

\label{tab:np_mfg_hyperparams}
\end{table}

\subsection{Image-to-image translation}

For image-to-image translation, we first train a deep variational autoencoder (VAE) to compress images into a latent space. 
The deep VAE is trained on the same setup following~\citep{rombach2022high} in a latent space of $\mathbb{R}^{8\times12\times12}$, enabling dynamic OT to be learned in the latent space across all methods. We use a constant learning rate of 4.5e-06 and a batch size of 32 for training VAE on both datasets. For the handbags–shoes dataset, a single VAE is trained for 100 epochs, using 80\% of the images for training and splitting the remaining 20\% evenly into validation and held-out sets. The held-out set is excluded from training and later used for FID evaluation. For the CelebA dataset, we train a single VAE for 50 epochs, adopting the standard training, validation, and test splits from \texttt{torchvision.datasets.CelebA}. The test set is excluded from training and used only for FID evaluation. Table~\ref{tbl:fid_vae} reports the FID between reconstructed images from the training set and ground-truth images from the held-out sets (handbags, shoes, CelebA male, and CelebA female), presenting the quality of the trained VAE.

We use a convolutional U-Net architecture~\citep{ronneberger2015u} for the classifier. The Swish activation function is used throughout the network, except for the final fully connected layer, where ReLU is used. For the velocity field, we use a stack of convolutional layers with ReLU activations. The time $t$ is concatenated with the input along the channel dimension, and the concatenated input is fed into each convolutional layer. To ensure a fair comparison, the same velocity field architecture is used for both our method and all baselines in the experiments. The detailed architectures of the classifier and velocity field are provided in Table~\ref{tbl:architecture_classifier} and Table~\ref{tbl:architecture_velocity_field}, respectively.

We initialize the velocity field in our method using flow matching on a linear interpolant for a specified number of training steps in both tasks. All hyperparameters of our method, except the velocity field and classifier architectures, are provided in Table~\ref{tbl:hyperparams_ours}. The hyperparameters of the baseline methods for the handbag-to-shoe and male-to-female translation tasks are listed in Table~\ref{tbl:hyperparams_baselines_shoebags} and Table~\ref{tbl:hyperparams_baselines_celeba}, respectively. All baselines are trained for 400K steps with a batch size of 512, ensuring comparability with the effective number of samples used in the particle optimization of our method.

\begin{table}[t]
\caption{FID between reconstructed images from the training set and
ground-truth images from the held-out sets on the images of Handbags, shoes, CelebA male, and CelebA female.}
\vspace{-10pt}
\begin{center}
\small
\begin{tabular}{cccc}
\toprule
Handbag & Shoes & Male &  Female \\
\midrule
 5.07  & 6.09 & 7.09 & 5.06 \\
\bottomrule
\end{tabular}
\end{center}
\label{tbl:fid_vae}
\end{table}

\begin{table}[t]
\caption{Architecture details of the classifier.}
\begin{center}
\begin{tabular}{llp{0.60\textwidth}}
\toprule
 & \textbf{Layer} & \textbf{Component} \\
\midrule
\multirow{15}{*}{\textbf{Encoding}}
  & Conv 1 & $3\times3$ \texttt{Conv2d}, channels=256, kernel size=3, stride=1, bias=True, padding=1 \\
  & Group Normalization 1 & num groups=4, channels=256 \\
  
  & Conv 2 & $3\times3$ \texttt{Conv2d}, channels=512, kernel size=3, stride=1, bias=True, padding=1 \\
  & Group Normalization 2 & num groups=32, channels=512 \\
  
  & Conv 3 & $3\times3$ \texttt{Conv2d}, channels=512, kernel size=3, stride=2, bias=True, padding=1 \\
  & Group Normalization 3 & num groups=32, channels=512 \\
  
  & Conv 4 & $3\times3$ \texttt{Conv2d}, channels=1024, kernel size=3, stride=1, bias=True, padding=1 \\
  & Group Normalization 4 & num groups=32, channels=1024 \\
  
  & Conv 5 & $3\times3$ \texttt{Conv2d}, channels=1024, kernel size=3, stride=1, bias=True, padding=1 \\
  & Group Normalization 5 & num groups=32, channels=1024 \\

\midrule

\multirow{14}{*}{\textbf{Decoding}}
  & Conv Transpose 5 & $3\times3$ \texttt{ConvTranspose2d}, channels=1024, kernel size=3, stride=1, bias=True, padding=1 \\
  & Group Normalization 5 & num groups=32, channels=1024 \\
  
  & Conv Transpose 4 & $3\times3$ \texttt{ConvTranspose2d}, channels=1024, kernel size=3, stride=1, bias=True, padding=1 \\
  & Group Normalization 4 & num groups=32, channels=1024 \\
  
  & Conv Transpose 3 & $4\times4$ \texttt{ConvTranspose2d}, channels=512, kernel size=4, stride=2, bias=True, padding=1 \\
  & Group Normalization 3 & num groups=32, channels=256 \\

  & Conv Transpose 2 & $3\times3$ \texttt{ConvTranspose2d}, channels=512, kernel size=3, stride=1, bias=True, padding=1 \\
  & Group Normalization 2 & num groups=32, channels=512 \\
  
  & Conv Transpose 1 & $3\times3$ \texttt{ConvTranspose2d}, channels=256, kernel size=3, stride=1, bias=True, padding=1 \\
\bottomrule
\end{tabular}
\end{center}
\label{tbl:architecture_classifier}
\end{table}

\begin{table}[h]
\caption{Architecture details of the velocity field.}
\begin{center}
\renewcommand{\arraystretch}{1.15}
\begin{tabular}{llp{0.60\textwidth}}
\toprule
 & \textbf{Layer} & \textbf{Component} \\
\midrule
\multirow{5}{*}{\textbf{Encoding}}
  & Conv 1 & $3\times3$ \texttt{Conv2d}, channels=64, stride 1, bias=True, padding=1 \\
  & Conv 2 & $3\times3$ \texttt{Conv2d}, channels=256, stride 1, bias=True, padding=1 \\
  & Conv 3 & $3\times3$ \texttt{Conv2d}, channels=512, stride 2, bias=True, padding=1 \\
  & Conv 4 & $3\times3$ \texttt{Conv2d}, channels=512, stride 1, bias=True, padding=1 \\
  & Conv 5 & $3\times3$ \texttt{Conv2d}, channels=1024, stride 1, bias=True, padding=1 \\
\midrule
\multirow{10}{*}{\textbf{Decoding}}
  & Conv Transpose 5 & $3\times3$ \texttt{ConvTranspose2d}, channels=1024, stride 1, bias=True, padding=1 \\
  
  & Conv Transpose 4 & $3\times3$ \texttt{ConvTranspose2d}, channels=512, stride 1, bias=True, padding=1 \\
  
  & Conv Transpose 3 & $4\times4$ \texttt{ConvTranspose2d}, channels=512, stride 2, bias=True, padding=1 \\
  
  & Conv Transpose 2 & $3\times3$ \texttt{ConvTranspose2d}, channels=256, stride 1, bias=True, padding=1 \\
  
  & Conv Transpose 1 & $3\times3$ \texttt{ConvTranspose2d}, channels=64, stride 1, bias=True, padding=1 \\
\bottomrule
\end{tabular}
\end{center}
\label{tbl:architecture_velocity_field}
\end{table}

\begin{table}[h]
\caption{Hyperparameters of our method in image-to-image translation experiments.}
\begin{center}

\begin{tabular}{lcc}
\toprule
\multirow{2}{*}{\textbf{Hyperparameter}} & \multicolumn{2}{c}{\textbf{Value}} \\
\cmidrule(lr){2-3} & \textbf{Handbag to shoe} & \textbf{CelebA male to female} \\
\midrule
Total outer loop & 20 & 50 \\
Classifier training batch size & 256  & 256  \\
Initial classifier training steps & 2000 & 2000  \\
Classifier learning rate & 0.001 & 0.001  \\
Classifier intermediate training frequency & 10  & 10  \\
Classifier intermediate training steps & 10  & 10 \\
Velocity field training batch size & 256 & 256  \\
Velocity field initialization training steps & 10000 & 10000 \\
Velocity field training steps & 1000 & 300 \\
Velocity field learning rate & 0.001  & 0.001  \\
Particle optimization batch size & 512  & 1024 \\
Particle optimization steps & 1000 & 300 \\
Particle optimization learning rate & 0.001 & 0.001 \\
Kinetic loss weight & 0.05 & 0.05 \\
Number of timesteps & 15 & 15 \\
\bottomrule
\end{tabular}

\end{center}
\label{tbl:hyperparams_ours}
\end{table}

\begin{table}[h]
\caption{Hyperparameters of baselines on the handbags to shoes image translation experiment.}
\begin{center}
\begin{tabular}{lcccc}
\toprule
\textbf{Hyperparameter} & \textbf{OT-CFM} & \textbf{Stochastic Interpolants} & \textbf{Rectified Flow} & \textbf{SB-CFM} \\
\midrule
Training steps & 400000 & 400000 & 400000 & 400000 \\
Learning rate & 0.001 & 0.001 & 0.001 & 0.001  \\
Batch size & 512 & 512 & 512 & 512  \\
Number of timesteps & 15 & 15 & 15 & 15 \\
\bottomrule
\end{tabular}
\end{center}
\label{tbl:hyperparams_baselines_shoebags}
\end{table}

\begin{table}[h]
\caption{Hyperparameters of baselines on the CelebA male to female image translation experiment.}
\begin{center}
\begin{tabular}{lcccc}
\toprule
\textbf{Hyperparameter} & \textbf{OT-CFM} & \textbf{Stochastic Interpolants} & \textbf{Rectified Flow} & \textbf{SB-CFM} \\
\midrule
Training steps & 400000 & 400000 & 400000 & 400000 \\
Learning rate & 0.001 & 0.001 & 0.001 & 0.001  \\
Batch size & 512 & 512 & 512 & 512  \\
Number of timesteps & 15 & 15 & 15 & 15 \\
\bottomrule
\end{tabular}
\end{center}
\label{tbl:hyperparams_baselines_celeba}
\end{table}

\end{document}